\documentclass{article}
\usepackage[utf8]{inputenc}

\usepackage[numbers]{natbib}
\usepackage{caption}
\usepackage{subcaption}
\usepackage{url}
\usepackage{amssymb}
\usepackage{graphicx}
\usepackage{amsmath}
\usepackage{color}
\usepackage{listings}
\usepackage[framed, numbered]{matlab-prettifier}
\usepackage[T1]{fontenc}
\usepackage{float,graphicx,epsfig,amsmath,graphpap,matlab-prettifier, bm}
\usepackage[vmargin=3cm, hmargin=3cm]{geometry}
\usepackage[noend]{algpseudocode}
\usepackage{algorithm}
\usepackage{amsthm}
\usepackage{url}
\usepackage{todonotes}

\newtheorem{theorem}{Theorem}[section]
\newtheorem{corollary}{Corollary}[section]

\theoremstyle{definition}
\newtheorem{definition}{Definition}[section]

\newtheorem{lemma}{Lemma}[section]

\newcommand{\CA}{\mathcal{A}}
\newcommand{\R}{\mathcal{R}}
\newcommand{\sign}{\mathrm{sign}}
\newcommand{\maj}{\mathrm{maj}}

\newcommand{\Bin}{\mathrm{Bin}}
\newcommand{\E}{\mathrm{E}}
\newcommand{\Pb}{\mathrm{P}}
\newcommand{\Cov}{\mathrm{Cov}}
\newcommand{\Var}{\mathrm{Var}}
\newcommand{\cov}{\mathrm{Cov}}
\newcommand{\I}{\mathrm{I}}

\title{Noise Sensitivity and Stability of Deep Neural Networks for Binary Classification}

\author{Johan Jonasson
        \texttt{jonasson@chalmers.se} \\ 
        Chalmers University of Technology and University of Gothenburg \\
        Gothenburg, Sweden \\
        \\
        Jeffrey E. Steif
        \texttt{steif@chalmers.se} \\ 
        Chalmers University of Technology and University of Gothenburg \\
        Gothenburg, Sweden \\
        \\
        Olof Zetterqvist 
        \texttt{olofze@chalmers.se} \\ 
        Chalmers University of Technology and University of Gothenburg \\
        Gothenburg, Sweden
        }

\author{Johan Jonasson \thanks{jonasson@chalmers.se}
        \and
        Jeffrey E. Steif \thanks{steif@chalmers.se}
        \and
        Olof Zetterqvist \thanks{olofze@chalmers.se} \and \\ Chalmers University of Technology and University of Gothenburg,
        Gothenburg, Sweden
        }

\date{}
\begin{document}

\maketitle

\begin{abstract}
A first step is taken towards understanding often observed non-robustness phenomena of deep neural net (DNN) classifiers. This is done from the perspective of Boolean functions by asking if certain sequences of Boolean functions represented by common DNN models are noise sensitive or noise stable, concepts defined in the Boolean function literature. Due to the natural randomness in DNN models, these concepts are extended to annealed and quenched versions. Here we sort out the relation between these definitions and investigate the properties of two standard DNN architectures, the fully connected and convolutional models, when initiated with Gaussian weights.  
\end{abstract}

\section{Introduction} \label{sec:introduction}

The driving question of this paper is how robust a typical binary neural net classifier is to input noise, i.e.\ for a typical neural net classifier and a typical input, will tiny changes to that input make the classifier change its mind? When asking this, we take inspiration from phenomena observed for deep neural networks (DNN) used in practice and use that inspiration to give mathematically rigorous answers for some simple DNN models under one (of several possible) reasonable interpretations of the question. It is not a prerequisite for the reader to be familiar with DNNs to find the topic interesting and any Machine Learning lingo will be explained shortly. 

DNNs have shown results that range from good to staggering in many different data-driven areas, e.g.\ for prediction and classification. One of many reasons for this is that with sufficiently large models, neural networks can approximate any function \cite{cybenko1989approximation}. 
However, there is much to be discovered about these black box models, two concerns being about robustness and optimal model design. Studies have shown that DNNs are vulnerable to various attacks, where adding small noise to inputs can lead to significant differences in the output \cite{goodfellow2014explaining,koh2017understanding}. For example, an image that is clearly of a fish which a DNN classifier also strongly believes is a fish can be such that only changing it by a tiny amount of random noise suddenly makes the classifier assign high probability to that it is now a dog.
This raises the question of how stable DNN models tend to be under small perturbations such as these.
Obviously, any non-trivial classification function must have the property that for some input, only a tiny amount of change leads to a different output. However, how typical is an input $\omega$ such that the output $f(\omega)$ changes from tiny amounts of change in $\omega$? In this formulation, one can clearly interpret the word ``typical'' in many different ways and also consider many different ways of defining what a tiny change is. 

\smallskip

To take some small but rigorous steps towards answers, we will in this paper focus on the setting where the input into the DNN is a vector $\omega$ of binary bits: $\omega \in \{-1,1\}^n$, and the output is binary classification, $f(\omega) \in \{-1,1\}$. 

This point of view is not new and can be seen in Boolean networks. Here some research has been done for different noise settings, but to the best of our knowledge, it seems that none of it is close to what we propose, and little is strictly rigorous. To mention a few, \citep{peixoto2009noise} and \citep{xiao2007impact} work in a Boolean network setting. In terms of DNNs, one is then considering a recurrent neural network model with the exact same weight matrix at each layer. The first of these papers considers small changes in the input, and the second paper considers small changes in the network structure. Both conclude that the final output (i.e.\ the fixed point) is robust to these changes.

The perhaps most natural way to talk about a ``typical input'', and the one that we are going to adopt, is to consider an input generated at random from a given probability distribution on $\{-1,1\}^n$ (e.g.\ if we are considering one of the standard benchmark problems of classifying handwritten digits, we would e.g.\ consider inputting handwritten digits drawn from a probability distribution reflecting how people actually write digits). Then one asks what the probability is that the input is such that the DNN model changes its classification by changing the input in a tiny way.

\smallskip

\smallskip
 
We prove results for a fully connected DNN architecture with input noise which are valid for arbitrary probability distributions over the input and the noise (as long as the noise with high probability actually produces a change of at least one input bit). However, since the concepts studied are usually understood to assume uniform input distribution and pure noise, i.e.\ each input bit changes with some tiny predetermined probability independently over the bits, the presentation will be made under these conditions. However, in Sections \ref{sec:uncorrelated} and \ref{sec:correlated}, it becomes apparent that the sensitivity properties for fully connected DNN models hold under the most general conditions possible on the input distribution and the noise as will be commented on there. The later section's results rely however on uniform input and pure noise.

In summary, we intend to analyse robustness of DNNs to noise from the perspective of Boolean functions, i.e.\ to consider those feed-forward DNNs that represent functions with input in $\{-1,1\}^n$ and output in $\{-1,1\}$ and analyse how sensitive these are to small random noise to random input. 
Doing this, we find ourselves in, or at least very close to, the setting of the research field of noise sensitivity and noise stability of Boolean functions, concepts introduced in \citep{benjamini1999noise}. The standard references nowadays to these concepts are the textbooks \cite{garban2014noise,o2014analysis}. We will return with exact definitions and extensions of the concepts shortly, but in short, noise sensitivity means what we already said: a Boolean function $f$ is noise sensitive if for large $n$, $\omega \in \{-1,1\}^n$ uniformly random and $\eta$ that differs from $\omega$ by changing a tiny random amount of randomly chosen bits, then $f(\omega)$ and $f(\eta)$ are virtually uncorrelated. One says that $f$ is noise stable if such tiny changes are very unlikely to change the output of $f$. (Clearly, a rigorous definition must be in terms of asymptotics as $n \rightarrow \infty$).
The question in focus now becomes
\[\mbox{Is a Boolean function represented by a given DNN noise sensitive or noise stable?}\]

\smallskip

To further restrict the setting, the activation function will at all layers be assumed to be the sign function, and the linear transformation at a given node will always be without a bias term. Neither of these restrictions is common in practice, of course. Still, one can at least arguably claim that since the idea of feeding the output of a neuron into an activation function is to decide if that neuron fires or not, the sign function is ``the ideal'' activation (but, of course, not used in practice because of the difficulty in training). 

All in all, precisely and in a way that explains the DNN lingo, each model we consider will be such that there is a given Boolean function $h$, a given so called depth $T$ and given so called layer sizes $n_1,n_2,\ldots,n_T$. With those given, we are considering Boolean functions $f$ on Boolean input strings $\omega \in \{-1,1\}^n = \{-1,1\}^{n_0}$ that can be expressed by a choice of matrices $W_1,\ldots,W_T$, $W_t \in \mathbf{R}^{n_t \times n_{t-1}}$ for $t=1,\ldots,T$ and for $\omega=\omega_0 \in \{-1,1\}^n$ taking
\begin{equation} \label{eq:def_neural_net_part1}
\omega_t=\sign(W_t\omega_{t-1}), \, t=1,\ldots,T
\end{equation}
and
\begin{equation} \label{eq:def_neural_net_part2}
    f(\omega)=h(\omega_T)
\end{equation}
where the sign is taken point-wise. The most common $h$ is of the form $\sign(\mathbf{w}\,\omega_T)$ for some row vector $\mathbf{w}$, i.e.\ a prediction made from standard logistic regression on $\omega_T$. 

\smallskip

As already stated and from what is apparent from the restrictions made, i.e.\ Boolean input, sign activation functions, no bias terms and, in particular, uniform probability measure over inputs and noise, makes the setting fairly far removed from practical settings. Moreover, we will not consider models that have been fit to data in any way other than a loose motivation for one of our model choices when it comes to modelling randomness in the coefficients of the $W_t$-matrices. 

To consider the matrices as random is natural when taking inspiration from DNNs in practice, since when one trains the model to fit with data, i.e.\ minimise the loss function at hand, one usually starts the optimisation algorithm by taking the initial coefficients to be random, often i.i.d.\ normal. Also during training, further randomness is often brought into the picture by the use of stochastic gradient descent. In the end of course, the training algorithm converges, but since there are usually many local minima for the loss function, the randomness in the coefficients at the start makes it random which local maximum one converges to. Furthermore, convergence is almost never reached and this is on purpose, since it is common practice to use early stopping, i.e.\ stop training well before convergence, to avoid overfitting. As already declared though, we will not consider any data and model fitting, only observe that any training of course produces correlation between the random weight matrix components and suggest and analyse a tractable model of such correlation. 

Hence, in summary, we regard this paper as mostly a contribution to the field of noise sensitivity/stability of Boolean functions inspired by an interesting and important  phenomenon of DNN prediction models, rather than to applied machine learning. Nevertheless, it is a first step towards an understanding of the non-robustness phenomena of DNNs and, to our knowledge, the first strictly rigorous contribution.

\smallskip 

Observe that when considering the $W_t$'s as random, the precise predictor function $f$ that comes out of it is in itself a random object. This is not the case in the field of noise sensitivity/stability and one can now ask two different things: (i) will the predictor be noise sensitive when taking both the randomness in the predictor itself and the randomness in the input and the noise into account?, (ii) will the random predictor after the weights have been drawn, with high probability become noise sensitive in the usual sense? This leads us to extend the standard definitions of noise sensitivity and noise stability to also encompass these aspects.

\smallskip
 
The paper is structured as follows.
Section \ref{sec:sens/stab} focuses on the relevant concepts and states the relations between them. The remaining sections each focus on selected examples of models of the family of DNN architectures given by (\ref{eq:def_neural_net_part1}) and (\ref{eq:def_neural_net_part2}) with natural assumptions on the randomness of the weights.
\begin{itemize}
    \item[] {\bf Section \ref{sec:uncorrelated}:}  Fully connected DNN with $T_n$ layers of equal width $n$. All weights are assumed to be standard normal and mutually independent. This is a standard configuration of the DNN at the start of training. We prove that as soon as $\lim_{n \rightarrow \infty} T_n = \infty$, the weights will with very high probability be such that the resulting classifier is very strongly sensitive to perturbations of the input no matter the input distribution and noise distribution. If $T_n$ is bounded, the resulting DNN will produce a noise stable classifier.
    
    \item[] {\bf Section \ref{sec:correlated}:}  All weights are once again standard normal, but some of them are now correlated: the columns of each $W_t$ are multivariate normal with all correlation being $\rho_n$. The columns are mutually independent within and across the $W_t$'s. We show that with $\rho_n$ converging to $1$ sufficiently fast, the resulting DNN becomes noise stable with high probability. If $\rho_n$ converges to $1$ slowly enough, the resulting classifier is with high probability strongly sensitive to perturbations.
    
    \item[] {\bf Section \ref{sec:conv}:} $2k+1$-majority on ``$2k+1$-trees with overlaps'', i.e.\ the graph where the vertices/neurons in each generation share some children, see Figures \ref{fig:stride1}, \ref{fig:stride2}. 
    It is proved that if the number of children shared by two parents next to each other is $2k$ (corresponding to stride $s=1$), then the resulting Boolean function is noise stable, whereas if the number of children shared is less than $2k$ (stride $s \geq 2$), then we get a noise sensitive Boolean function. 

    These models are convolutional neural nets where, in machine learning language, each filter represents a regular majority of the input bits. It will be observed that the results easily extend to the case where the weights of the filters are random under the only condition that there is at least some chance that a filter represents regular majority. 
\end{itemize}

In the sections on the fully connected models, analysis of a certain Markov chain on $\{0,1,\ldots,n-1,n\}$, which is symmetric around $n/2$ and absorbs in $0$ and $n$, plays a central role. We believe that the structure of this Markov chain makes it interesting in its own right.

\section{Different notions of noise sensitivity and noise stability} \label{sec:sens/stab}

Let $\omega \in \{\pm 1\}^n$ be an i.i.d.\ ($1/2$,$1/2)$ Boolean row vector and let $\{f_n\}$ be a sequence of Boolean functions from $\{\pm 1\}^n $ to $\{\pm 1\}$. Additionally let $\omega^\epsilon$ be a Boolean row vector such that $\omega^\epsilon(i) = \omega(i)$ with probability $1 - \epsilon$ and $\omega^\epsilon(i) = -\omega(i)$ with probability $\epsilon$ independently for different $i$. View $\omega^\epsilon$ as a small perturbation of $\omega$. In this context, we can now define noise sensitivity and noise stability of a sequence of Boolean functions $f_n$. In \cite{benjamini1999noise} these are defined as

\begin{definition}
    \label{noise_sensitive}
    The sequence $\{f_n\}$ is \textbf{noise sensitive} if for every $0 < \epsilon \leq 1/2$, \[ \lim_{n \rightarrow \infty} \Cov\left(f_n(\omega),f_n(\omega^\epsilon)\right) = 0\]
\end{definition}
\begin{definition}
    \label{noise_stable}
    The sequence $\{f_n\}$ is \textbf{noise stable} if \[ \lim_{\epsilon \rightarrow 0} \sup_n \Pb(f_n(\omega) \neq f_n(\omega^\epsilon) ) = 0.\]
\end{definition}

The definition of noise stability is easily seen to be equivalent with the condition $\lim_{\epsilon \rightarrow 0} \limsup_n \Pb(f_n(\omega) \\ \neq f_n(\omega^\epsilon) ) = 0$. An example of a noise stable sequence is the weighted majority functions 

\[\maj_{\theta^{(n)}}(\omega) = \sign \left(\sum_{j=1}^n \theta_j^{(n)} \omega(j)\right),\]

where $\theta_1^{(n)},\hdots, \theta_n^{(n)}$ are arbitrary given constants
\cite{peres2021noise}. An example of a noise sensitive sequence is the parity functions \citep{garban2014noise}, 

\[\mathrm{par}(\omega) = \prod_{j=1}^n \omega_j.\]

For noise sensitivity, there is a more general and often much stronger concept, which will turn out to be interesting here since most fully connected and sufficiently deep neural nets will turn out to be noise sensitive in a very strong way. If a sequence of functions $\{f_n\}$ is noise sensitive as defined above, then one can always find a sequence $\epsilon_n \leq 1/2$ tending to $0$ with $n$ slowly enough such that
\[\lim_{n \rightarrow \infty} \Cov(f_n(\omega),f_n(\omega^{\epsilon_n})) = 0.\]
Since $\Cov(f_n(\omega),f_n(\omega^\epsilon))$ is well known to be decreasing in $\epsilon$ on $[0,1/2]$, the faster $\epsilon_n$ can be taken to decrease, the stronger the statement. This leads to the following definition,

\begin{definition}
    \label{QNS}
    Let $\epsilon_n \leq 1/2$ be non-increasing in $n$ with $\epsilon_n \rightarrow 0$. The sequence $\{f_n\}$ is \textbf{quantitatively noise sensitive (QNS) at level $\{\epsilon_n\}$} if, \[ \lim_{n \rightarrow \infty} \Cov\left(f_n(\omega),f_n(\omega^{\epsilon_n})\right) = 0.\]
\end{definition}

The definitions of noise sensitivity and stability are by now standard when describing properties of deterministic sequences of Boolean functions. However, when dealing with randomness within the functions themselves we need a more general definition. This more general setup occurs in  Boolean neural networks, where the network parameters $\Theta$ can be seen as random elements (usually depending on randomness in what training data is presented to the network and the initial values of the parameters before training). 
Let $\mathcal{F}_n$ be the set of all Boolean function from $\{\pm 1\}^n$ to $\{\pm 1\}$ and let $\pi_n$ be a arbitrary probability measure on $\mathcal{F}_n$. Recall that for $0 \leq \epsilon \leq 1/2$ and each function $f$, we have $\Cov_{\omega,\omega^\epsilon}(f(\omega),f(\omega^\epsilon)) \geq 0$, and hence $\Cov_{f,\omega,\omega^\epsilon}(f(\omega),f(\omega^\epsilon)) \geq 0$. We can then define both quenched and annealed versions of noise sensitivity and noise stability as follows.

\begin{definition}
    \label{quenched_noise_sensitive}
    $\pi_n$ is \textbf{quenched QNS at level $\{\epsilon_n\}$} if for every $\delta > 0$ and $0 < \epsilon_n \leq 1/2$, there is an $N$ such that for all $n \geq N$ 
    \[\pi_n\{f_n \: : \: \Cov_{\omega,\omega^\epsilon}(f_n(\omega),f_n(\omega^{\epsilon_n})) \leq \delta\} \geq 1 - \delta.\]
\end{definition}    

\begin{definition}
    \label{annealed_noise_sensitive}
    $\pi_n$ is \textbf{annealed noise QNS at level $\{\epsilon_n\}$} if for every $0 <\epsilon_n \leq 1/2$
    \[\lim_{n \rightarrow \infty} \Cov_{f_n,\omega,\omega^{\epsilon_n}}(f_n(\omega),f_n(\omega^{\epsilon_n})) = 0.\] 
\end{definition}

\begin{definition}
    \label{quenched_noise_stable}
    $\pi_n$ is \textbf{quenched noise stable} if for every $\delta$ there is an $\epsilon > 0$ such that for all $n$,
    \[\pi_n\{ f_n : \Pb_{\omega,\omega^\epsilon} (f_n(\omega) \neq f_n(\omega^\epsilon)) < \delta \} \geq 1 - \delta\]
\end{definition}    

\begin{definition}
    \label{annealed_noise_stable} 
    $\pi_n$ is \textbf{annealed noise stable} if
    \[\lim_{\epsilon \rightarrow 0} \sup_{n} \Pb_{f_n,\omega,\omega^\epsilon}(f_n(\omega)\neq f_n(\omega^\epsilon)) = 0\]
\end{definition}

This notion of quenched noise sensitivity has arisen elsewhere;
see \citep{ahlberg_2014,AHLBERG2016889,vanneuville2018quantitative,10.1214/21-EJP712}. Notice that if $\pi_n$ has support on only one Boolean function, then these definitions are equivalent to the usual ones in Definition \ref{noise_sensitive} and \ref{noise_stable}. In Theorem \ref{th:relations_between_annealed_and_quenched} we show which relations there are between these definitions.

\begin{theorem}

    \label{th:relations_between_annealed_and_quenched}
    Let $\mathcal{F}_n$ be the set of all Boolean functions on $\{-1,1\}^n \rightarrow \{-1,1\}$ and let $\pi_n$ be a probability measure on $\mathcal{F}_n$. Then the following are true
    
    \begin{enumerate}
        \item[(i)] $\{\pi_n\}$ is annealed QNS at level $\{\epsilon_n\}$ iff $\{\pi_n\}$ is quenched QNS at level $\{\epsilon_n\}$ and $\Var_{f_n}(\E_\omega[f_n(\omega)]) \rightarrow 0$ as $n \rightarrow \infty$.
        \item[(ii)] $\{\pi_n\}$ is annealed noise stable iff $\{\pi_n\}$ is quenched noise stable. 
    \end{enumerate}

\end{theorem}   
\begin{proof}
To prove the first statement we use the conditional covariance formula that for any random variables $X,Y$ and $Z$
\[\cov(X,Y) = \E [\cov(X,Y|Z)] + {\cov(\E[X|Z],\E[Y|Z])}\]
to observe that 
\begin{equation}
\label{eq:covariance_formula}
\cov_{f,\omega,\omega^\epsilon}(f(\omega),f(\omega^{\epsilon_n})) = \E_f \left[\cov_{\omega,\omega^{\epsilon_n}}(f(\omega),f(\omega^{\epsilon_n})) \right]  + \cov_{f}(\E_\omega[f(\omega)],\E_{\omega^{\epsilon_n}}[f(\omega^{\epsilon_n})]).
\end{equation}
Now observe the following. First, the first term on the right is always non-negative. Secondly, since $\omega$ and $\omega^{\epsilon_n}$ are equal in distribution, the second term on the right is equal to $\Var_f(\E_\omega[f(\omega)])$. 

We now prove (i) starting with that quenched QNS at level $\{\epsilon_n\}$ and $\Var_f(\E_\omega[f_n(\omega)]) \rightarrow 0$ as $n \rightarrow \infty$ leads to annealed QNS at level $\{\epsilon_n\}$. Fix $\delta \in (0,1]$.
We know that there exists an $N$ such that for all $n > N$
\[\pi_n\{f : \Cov_{\omega,\omega^{\epsilon_n}}(f(\omega),f(\omega^{\epsilon_n})) \leq \frac{\delta}{4}\} \geq 1 - \frac{\delta}{4}\]
and
\[\Var_f(\E_\omega[f(\omega)]) < \frac{\delta}{2}.\]

Using (\ref{eq:covariance_formula}) this leads to
\[\cov_{f,\omega,\omega^{\epsilon_n}}(f(\omega),f(\omega^{\epsilon_n})) < \E_f\left[\frac{\delta}{4} \I_{\Cov_{\omega,\omega^{\epsilon_n}}(f(\omega),f(\omega^{\epsilon_n})) \leq \frac{\delta}{4}} + \I_{\Cov_{\omega,\omega^{\epsilon_n}}(f(\omega),f(\omega^{\epsilon_n})) > \frac{\delta}{4}}\right] + \frac{\delta}{2} < \delta\]
This proves the first direction of (i). For the other direction,
fix $\delta \in (0,1]$. 
Since $\pi_n$ is annealed QNS at level $\{\epsilon_n\}$ 
there exists an $N$ such that $ \forall n > N$
\[\cov_{f,\omega,\omega^{\epsilon_n}}(f(\omega),f(\omega^{\epsilon_n})) < \delta^2.\]
Now, using (\ref{eq:covariance_formula}) and the fact that $\cov_{\omega,\omega^{\epsilon_n}}(f(\omega),f(\omega^{\epsilon_n})) \geq 0$ for all $f$ it must be that 
\[\Var_f(\E_\omega[f(\omega)]) < \cov_{f,\omega,\omega^{\epsilon_n}}(f(\omega),f(\omega^{\epsilon_n})) < \delta^2 < \delta.\] 
Hence the variances converge to zero. Additionally, for such $n$ we have that $0 \leq \cov_{f,\omega,\omega^{\epsilon_n}}(f(\omega),f(\omega^{\epsilon_n})) \\ \leq \delta^2$. Now using Markov's inequality and (\ref{eq:covariance_formula}) we get 
\[\pi_n\{f : \cov_{\omega,\omega^{\epsilon_n}}(f(\omega),f(\omega^{\epsilon_n})) \geq \delta \} \leq \delta\]
which concludes (i). 

Next we prove (ii). We  start by showing that quenched noise stability leads to annealed noise stability.
    \noindent Fix $\delta > 0$. $\pi_n$ being quenched noise stable means that there exist an $\epsilon > 0$ such that for all $n$
    
    \[\pi_n\{ f : \Pb (f(\omega) \neq f(\omega^\epsilon)) < \frac{\delta}{2} \} \geq 1 - \frac{\delta}{2}.\]
    
    \noindent Hence 
    
    \begin{align*}
        &\Pb_{f,\omega,\omega^\epsilon}\left(f(\omega) \neq f(\omega^\epsilon)\right) = \E_f\left[\E_{\omega,\omega^\epsilon}[\I_{f(\omega) \neq f(\omega^\epsilon)}]\right] \leq \\
        &\E_f\left[\frac{\delta}{2} \I_{\E_{\omega,\omega^\epsilon}[I_{f(\omega) \neq f(\omega^\epsilon)}] \leq \frac{\delta}{2}} + \I_{\E_{\omega,\omega^\epsilon}[\I_{f(\omega) \neq f(\omega^\epsilon)}]> \frac{\delta}{2}}\right] < \delta.
    \end{align*}

    \noindent This proves the first part. Now we prove that annealed noise stable implies quenched noise stable.
    
    Fix $\delta > 0$ and pick an $\epsilon > 0$ sufficiently small such that $\Pb_{f,\omega,\omega^\epsilon}(f(\omega) \neq f(\omega^\epsilon)) < \delta^2$. Such an $\epsilon$ are guaranteed to exist since $\pi_n$ are annealed noise stable. Then due to Markov´s inequality \[\pi_n (f : \Pb_{\omega,\omega^\epsilon}(f(\omega) \neq f(\omega^\epsilon)) > \delta) < \delta\] which gives us quenched noise stable. This proves (ii).
\end{proof}

As seen from statement (ii), $\pi_n$ being annealed noise stable is equivalent to $\pi_n$ being quenched noise stable. Therefore we will further on only refer it as $\pi_n$ being noise stable.

\section{Random Boolean feed forward neural networks} \label{sec:FFNN}

In this section, we investigate a Boolean function structure $f(\omega,\Theta)$ inspired by feed-forward neural networks. Let $\omega_0 \in \{-1,1\}^n$ be the input bits as a column vector. Then we can recursively define $\omega_t = \sign(\theta_t \omega_{t-1})$ where $\theta_t \in R^{n \times n}$ and $\sign$ acts pointwise. In the deep learning literature, the $\theta$'s and $\mathrm{sign}$ would be referred to as the weights of $f$ and the activation function respectively. In the Neural Network literature, each $t$ corresponds to a layer where $\omega_t$ are seen as the bits, or nodes, at layer $t$. The iteration is done for $t=1,\ldots,T$ for some predetermined number $T=T_n$ giving us the final output $h(\omega_T)$ where $h=h_n$ is some Boolean function $\{-1,1\}^n \rightarrow \{-1,1\}$. A known fact is that a typical Neural Network can approximate any function to arbitrary accuracy as long as the amount of tunable parameters (a.k.a.\ weights), $\Theta = \{\theta_1, \hdots, \theta_T\}$, is large enough. In the Boolean setting, these models can still represent a huge number of Boolean functions. However, there are some limitations since no bias term is present in our Boolean Network, which typically is in Neural Networks in practice. Typically $\Theta$ is determined by some training algorithm based on the observation data, which maximises the likelihood of the model. This means that $\Theta$ is not deterministic since there is both randomness in the observed data and in the optimisation algorithm. Therefore we can consider a probability measure $\pi_n$ on all $f$ where the randomness comes from $\Theta$.

Here we consider cases where the $\theta_t$'s are independent and, for each $t$, the columns of $\theta_t$ are independent and jointly normal $N(0,\Sigma)$ with two different versions of $\Sigma$. These cases, which thus induce their respective measures $\{\pi_n\}$ on the set of Boolean function, are

\begin{enumerate}
    \item $\Sigma = \mathbf{I}_{n \times n}$.
    
    \item $\Sigma = \rho+(1-\rho)\mathbf{I}_{n \times n}$. A useful way to construct such $\theta_t$'s is to define $\theta_t(i,j) = \sqrt{\rho}\nu_t(j) + \sqrt{1 - \rho} \psi_t(i,j)$ where $\nu_t(i)$, $\psi_t(i,j) \sim N(0,1)$ all being independent. 
\end{enumerate}
The case $\rho =0$ represents a typical starting state for the network
before training. After some training, we expect the parameters of the network
to be dependent. The cases $\rho >0$ are examples of such dependence.
Of course, we do not expect that the true dependence after training
is represented in this way. Our assumptions should therefore be viewed
as a simplifying mathematical framework under which our theorems can
be proved. Somewhat related to this question of the behaviour of the parameters, in \cite{gripon2018improving}, a particular statistical structure for the values of intermediate layers in some convolution codes was discovered.
 
The following lemma is crucial.

\begin{lemma}
\label{th:tan_prop}
Let $x,y \in \{-1,1\}^n$ be column vectors such that $|i : x(i) \neq y(i)| = nv$ for some $v \in [0,1]$ and let $\theta$ be a random row vector such that $\theta \sim N(0,I_{n \times n})$. Then
\[\Pb\left( \sign(\theta x) \neq \sign(\theta y)\right) = \frac{2}{\pi} \arctan\left( \sqrt{\frac{v}{1-v}}\right)\]

\end{lemma}
\begin{proof}
    Let $C$ be the set of indices where $x$ and $y$ differ. Notice that $|C| = nv$. Consider the line segment between $x$ and $y$ defined as $\frac{y + x}{2} + \tau \frac{y - x}{2}$, $\tau \in [-1,1]$. Then $\sign(\theta x) \neq \sign(\theta y)$ if there is a solution to 
    \begin{equation}
    \theta^T \left(\frac{y + x}{2} + \tau \frac{y - x}{2} \right) = 0
    \end{equation}
    for some $\tau \in [-1,1]$.
    
    Now, let $A_{\Lambda,k} = \{j \: : \: j \in \Lambda, \: x(j) = k\}$ for some set $\Lambda$. Since both $x$ and $y$ are in the hypercube $\{-1,1\}^n$, the condition can be rewritten as 
    \begin{equation}
    \label{eq:tan_event}
    \left|\sum_{j \in A_{C^c,1}} \theta(j) - \sum_{j \in A_{C^c,-1}} \theta(j) \right| \leq \left|\sum_{j \in A_{C,1}} \theta(j) - \sum_{j \in A_{C,-1}} \theta(j)\right|.
    \end{equation}
    Define $X$ and $Y$ from the following equations \[\sqrt{nv}X = \sum_{j \in A_{C,1}}\theta(j) - \sum_{j \in A_{C,-1}}\theta(j)\] and \[\sqrt{n - nv}Y = \sum_{j \in A_{C^c,1}}\theta(j) - \sum_{j \in A_{C^c,-1}}\theta(j).\]
    Then it is easy to check that $X$ and $Y$ are standard independent normal and that the event in (\ref{eq:tan_event})
    can be rewritten as
    \[ -\sqrt{\frac{v}{1-v}}|X| \leq Y \leq \sqrt{\frac{v}{1-v}}|X|.\]
    Due to symmetry around $Y = 0$, this results in \[\Pb\left(-\sqrt{\frac{v}{1-v}}|X| \leq Y \leq \sqrt{\frac{v}{1-v}}|X| \right) = \int_{-\infty}^{\infty} \int_{-\sqrt{\frac{v}{1-v}}|x|}^{\sqrt{\frac{v}{1-v}}|x|}\frac{1}{2 \pi} e^{-(x^2 + y^2)/2} dydx \]\[
    = 2 \int_{-\arctan\left(\sqrt{\frac{v}{1-v}}\right)}^{\arctan\left(\sqrt{\frac{v}{1-v}}\right)} \int_{0}^{\infty}  \frac{r}{2 \pi} e^{-r^2/2} d\varphi dr = \int_{-\arctan\left(\sqrt{\frac{v}{1-v}}\right)}^{\arctan\left(\sqrt{\frac{v}{1-v}}\right)} \frac{1}{\pi} d\varphi= \frac{2}{\pi} \arctan\left(\sqrt{\frac{v}{1-v}}\right)\] where the second equality is due to the substitution $x = r \cos(\varphi)$ and $y = r \sin(\varphi)$. This concludes the proof.
\end{proof}

In both noise parameter settings $1$ and $2$, we have independence between the weights leading into a node. This means, according to Lemma \ref{th:tan_prop}, that given $\omega_{t-1}$ and $(\omega^\epsilon)_{t-1}$ differ at $nv$ bits the probability that $\omega_t(i)$ and $(\omega^\epsilon)_t(i)$ differ is $\frac{2}{\pi} \arctan(\sqrt{\frac{v}{1-v}})$ for all $i$. The difference between the two parameter settings is that in $2$ the output bits of a layer are usually correlated. 
Also, due to symmetry, the probability of a disagreement at a fixed point, i.e. $\omega_t(i) \neq (\omega^\epsilon)_t(i)$, at layer $t$ depends only on the number of disagreements between $\omega_{t-1}$ and $(\omega^\epsilon)_{t-1}$ and not where they disagree. This means that the number of disagreements at layer $t$, which we denote by $D_t=D_t^\epsilon = D_t^{\epsilon_n}$, can in both cases be seen as a Markov chain with $n+1$ states, where $D_t = 0$ and $D_t = n$ are absorbing states. Notice that $D_t$ depends on the initial noise $\epsilon$ since $D_0$ corresponds to the number of bit disagreements created by the initial noise. However, for the sake of lighter notation we will not have a specific suffix showing this dependence.  

\subsection{Uncorrelated networks} \label{sec:uncorrelated}
In the uncorrelated case, $D_t$ is binomially distributed according to $(D_t|D_{t-1} = nv) \sim \Bin(n,g(v))$ where $g(v) :=  \frac{2}{\pi}\arctan\left(\sqrt{\frac{v}{1-v}}\right)$. Note that if $\epsilon_n$ is of order $1/n$ or lower, then $\Pb(D_1=0)$ does not get to $0$ at the same time as $P(D_1>n/2) \rightarrow 0$, which by symmetry implies that there cannot be QNS at that level. Hence one must have at least $n\epsilon_n \rightarrow \infty$ for QNS to be possible. 

The following theorem almost entirely considers the sensitivity properties of $f_{n,T_n}$ and shows that under very mild assumptions for $T_n$ larger than a specified function of $n$, $\{f_{n,T_n}\}$ is QNS in this the strongest possible sense (in particular that $f_{n,T_n}$ is noise sensitive in the original sense as soon as $T_n \rightarrow \infty$). Indeed, since the distribution of $\theta_t$ is such that even if we know $\omega$ and $\omega^\epsilon$, any trace of that knowledge is forgotten after the first layer, $f_{n,T_n}$ has very strong sensitivity properties even for fixed input and fixed noise. What we mean precisely with this is formulated separately in Theorem \ref{th:stronger_noise_sensitivity_deep_networks}. Note in particular part (ii) implies that the input and noise distribution do not need to be uniform but can indeed be taken to reflect what is realistic in the application at hand, e.g.\ a distribution over images of real handwritten digits.

Parts of Theorem \ref{th:stronger_noise_sensitivity_deep_networks} are clearly stronger than their counterparts in Theorem \ref{th:noise_sensitivity_deep_networks}, but we find it natural to state and prove the weaker statements first and then extend them by pointing out the fairly minor extra observations that need to made in the proof.   

\begin{theorem}
    \label{th:noise_sensitivity_deep_networks}
    Consider the above fully connected network with i.i.d.\ normal entries in each $\theta_t$ and i.i.d.\ $\theta_t$'s. Let $1/2 \geq \epsilon_n \downarrow 0$ be such that $n\epsilon_n \rightarrow \infty$ and let $\lim_{n \rightarrow \infty}K_n/\log(1/\epsilon_n) = \infty$. Then 
    \begin{itemize}
        \item[(i)] if $\lim_{n \rightarrow \infty} b_n =0$, $\lim_{n \rightarrow \infty} nb_n=\infty$ and $T_n \in [K_n,e^{b_n n}]$, then for any Boolean functions $\{h_n\}$, the resulting $\{f_{n,T_n}\}$ is annealed QNS at level $\{\epsilon_n\}$ with respect to $\{\pi_n\}$,
        \item[(ii)] if the $h_n$'s are odd and $T_n \geq K_n$, then $\{f_{n,T_n}\}$ is annealed QNS at level $\{\epsilon_n\}$,
        \item[(iii)] if $T_n \geq K_n$ then for any $\{h_n\}$, $\{f_{n,T_n}\}$ is quenched QNS at level $\{\epsilon_n\}$,
        \item[(iv)] there are Boolean functions $\{h_n\}$ such that for $T_n$ growing sufficiently fast with $n$, $\{f_{n,T_n}\}$ is not annealed noise sensitive.
    \end{itemize}
    In addition,
    \begin{itemize}
    \item[(v)] if $h_n$ is noise stable and $T_n$ is bounded, then $\{f_{n,T_n}\}$ is annealed (and hence quenched) noise stable.
    \end{itemize}
\end{theorem}

\noindent {\bf Remark.} If one randomly chooses a sequence of Boolean functions uniformly among {\it all} Boolean functions, it is known that the sequence will asymptotically almost surely
be noise sensitive; see Exercise 1.14 in \cite{garban2014noise}. While the Boolean functions arising in Theorem \ref{th:noise_sensitivity_deep_networks} here are also random, they have a very specific form.

\begin{theorem}
    \label{th:stronger_noise_sensitivity_deep_networks}
    Consider the above fully connected network with i.i.d.\ normal entries in each $\theta_t$ and i.i.d.\ $\theta_t$'s. Let $1/2 \geq \epsilon_n \downarrow 0$ be such that $n\epsilon_n \rightarrow \infty$ and let $\lim_{n \rightarrow \infty}K_n/\log(1/\epsilon_n) = \infty$. Then 
    \begin{itemize}
        \item[(i)] under the assumptions of either (i) or (ii) in Theorem \ref{th:noise_sensitivity_deep_networks}, for any fixed $\omega,\eta \in \{-1,1\}^n$ with $\omega \not \in \{\eta,-\eta\}$,
        \[\lim_{n \rightarrow \infty} \Cov_{\Theta}(f_{n,T_n}(\omega),f_{n,T_n}(\eta)) = 0.\]
        \item[(ii)] under the assumptions of either (i) or (ii) in Theorem \ref{th:noise_sensitivity_deep_networks}, for any probability measure $\mathbf{Q}_n$ on $\{-1,1\}^n \times \{-1,1\}^n$ such that $\lim_{n \rightarrow \infty} \mathbf{Q_n}(\{(\omega,\eta):\eta \in \{\omega,-\omega\}\})=0$.
        \[\lim_{n \rightarrow \infty} \Cov_{\mathbf{Q_n},\Theta}(f_{n,T_n}(\omega),f_{n,T_n}(\eta)) = 0.\]
        \item[(iii)] Assume that $h_n$ is odd and fix any $k \in \{1,2,\ldots,n-1\}$ and $\delta>0$. Fix also $\omega \in \{-1,1\}^n$ and let $M_k=M^{(n)}_k(\omega)$ be the number of $\eta$ with $\eta(i) \neq \omega(i)$ for exactly $k$ indexes $i$, such that $f_{n,T_n}(\eta) \neq f_{n,T_n}(\omega)$. Then for $T_n \geq K_n$,
        \[\lim_{n \rightarrow \infty}\Pb\left(\frac{M_k}{\binom{n}{k}} \not \in \left(\frac{1-\delta}{2},\frac{1+\delta}{2}\right)\right) = 0.\]
    \end{itemize}
\end{theorem}

\begin{proof}[Proof of Theorem \ref{th:noise_sensitivity_deep_networks}]
Recall $D_t=D_t^{\epsilon_n}$ from above. In the sequel in (i)-(iii), to not burden the notation, write just $\epsilon$ for $\epsilon_n$ with the understanding that $\epsilon=\epsilon_n$.

The conditional distribution of $(\omega_T,(\omega^{\epsilon})_T)$ given $\mathcal{F}_{T-1}:=\sigma(\omega_0,(\omega^\epsilon)_0,\theta_1,\ldots,\theta_{T-1})$ equals that of $(\omega_0,(\omega^{g(D_{T-1}/n)})_0)$. In other words; to determine the distribution of $(\omega_T,\omega^\epsilon_T)$ given $\mathcal{F}_{T-1}$, we only need to know $D_{T-1}$. 
Consequently, $\E[f(\omega_0)f((\omega^\epsilon)_0)|D_{T-1}=d] = \E[h(\omega_T)h((\omega^\epsilon)_T)|D_{T-1}=d] = \E[h(\omega_0)h((\omega^{g(d/n)})_0)]$. Let us study how $D_t$ behaves, started from $D_0 \sim \Bin(n,\epsilon)$.
First observe that $g(v)/v$ is decreasing on $(0,1/2)$. This holds as $g(0)=0$ and an easy computation shows $g''(v)<0$ for $v \in [0,1/2]$. Since $g'(1/2)=2/\pi$, it follows from Taylor's formula that
\[g\left(\frac12-\delta\right) > \frac12-\frac23 \delta\]
for sufficiently small $\delta>0$. Fixing such a $\delta$, it then follows that $g(d/n) > (1+2\delta/3)d/n$ whenever $1 \leq d<n/2-\delta n$. This means by Chernoff bounds and the fact that $g$ is increasing that there is $\kappa=\kappa(\delta,\epsilon)>0$ such that
\begin{itemize}
    \item for $\epsilon n/2 < d < (1/2-\delta)n/2$, $\Pb(D_{t+1}<(1+2\delta/3)d|D_t=d) < e^{-\kappa \sqrt{n}}$,
    \item for $d > (1/2-\delta)n$, $\Pb(D_{t+1}<(1/2-\delta)n|D_t =d) < e^{-\kappa n}$.
\end{itemize}
(Here the $\sqrt{n}$ in the exponent in the first point follows from that $g(1/n)$ is of order $1/\sqrt{n}$ and $\epsilon_n > 1/n$ for $n$ large.)
Combining these two points and the symmetry of $g$ around $1/2$,
\[\Pb\left(\exists t \in \left[\frac{\log\frac{1-2\delta}{\epsilon}}{\log\left(1+\frac{2\delta}{3}\right)}, T\right] : D_t \not\in \left[\left(\frac12-\delta \right)n,\left(\frac12+\delta \right)n \right]\right) < Te^{-\kappa n}.\]

To prove (i), let $L_n=e^{b_n n}$. Take now $T \in \left[1+\frac{\log\frac{1-2\delta}{\epsilon_n}}{\log\left(1+\frac{2\delta}{3}\right)},L_n\right]$. Conditionally on $D_{T-1} = d \in [(1/2-\delta)n,(1/2+\delta)n]$, we have that $(\omega_T,(\omega^\epsilon)_T)$ equals in distribution $(\omega_0,(\omega^\alpha)_0)$ for $\alpha=g(d/n) \in [1/2-\delta,1/2+\delta]$.
Equation $(4.2)$ in Section $4.3$ in \citep{garban2014noise} implies that for $\rho > 0$, there exists $\delta >0$ so that for all Boolean functions $h$ and for all $\alpha \in [1/2 - \delta, 1/2 + \delta]$, \[\Cov(h(\omega),h(\omega^\alpha)) < \rho.\]

Hence for $\rho>0$ and $\delta>0$ sufficiently small it follows from the above that
$\Cov(f(\omega_0),f((\omega^\epsilon)_0)|D_{T-1} \\ \in [(1/2-\delta)n,(1/2+\delta)n]]) < \rho/2$.
Since $\Pb(D_{T-1} \in [(1/2-\delta)n,(1/2+\delta)n]) > 1-L_n e^{-\kappa n} > 1-\rho/8$ for large $n$, we get
\[\Cov(f(\omega_0),f((\omega^\epsilon)_0)) < \rho.\]

To see this, let $B$ be the event $\{D_{T-1} \in [(1/2-\delta)n,(1/2+\delta)n]\}$ and recall that
$\Cov(f(\omega_0),f((\omega^\epsilon)_0)) \\ =  \E[\Cov(f(\omega_0),f((\omega^\epsilon)_0)|I_B)] + \Var(\E[f(\omega_0)|I_B])$. Since $\Pb(B^c)<\rho/8$, it is easy to see that the first term is bounded by $\Cov(f(\omega_0),f((\omega^\epsilon)_0)|B)+\rho/4$ and the second term is bounded by $\rho/2$.

This proves that for some constant $K=K(\rho,\epsilon)$, $T \in [K,L_n]$ and $n$ sufficiently large, $\Cov(f(\omega),f(\omega^\epsilon)) \\ <\rho$.
In particular if $K_n \rightarrow \infty$
and $T \in [K_n,L_n]$, then $f$ is annealed QNS. This proves (i). 

\smallskip 

Next consider (ii) and (iii). This amounts to considering what happens to $D_t$ in the long run. We have argued that with overwhelming probability, $D_t$ will quickly approach $n/2$ and stay there for a very long time. However, after an even longer time, $D_t$ will end up in one of the absorbing states $0$ or $n$. Then the above argument for annealed QNS does not hold (since it is no longer true that $D_t \in[(1/2-\delta)n,(1/2+\delta)n]$ with high probability).

Let $\{D'_t\}$ be a copy of $\{D_t\}$ but started with $D'_0 \sim \Bin(n,1/2)$. We can couple $(D_0,D'_0)$ so that $D_0 \leq D'_0$, so do that. 

Since $g$ is increasing, the distribution of $D_t$ given $D_{t-1}=d$ is increasing in $d$. Hence, since we have coupled so that $D_0 \leq D'_0$, there is a further coupling $(D,D')$ such that $D'_t \geq D_t$ for every $t$. Use such a coupling from now on and note that a consequence is that if at some point $D'_t=D_t$, then also $D'_s=D_s$ for all $s \geq t$. By the above with $t_0=1+\log(4\epsilon)/\log(6/7)$, $P(\forall t \in [t_0,L_n]: D_t,D'_t \in [n/4,3n/4]) > 1-L_ne^{-\kappa n} > 1-e^{-\kappa n/2}$ for large $n$. For all $d,d' \in [n/4,3n/4]$ and $d \leq d'$,
\begin{align*}
    \E[D'_t-D_t|D_{t-1}=d,D'_{t-1}=d'] &= n\left(g\left(\frac{d'}{n}\right)-g\left(\frac{d}{n}\right)\right)\\
    &< \frac34(d'-d),
\end{align*}
where the last inequality follows on observing that $g'(v)<3/4$ for $v \in [1/4,3/4]$. 
This gives for $t \in [t_0,L_n]$,
\begin{align*}
    \E[D'_t-D_t] &\leq ne^{- \kappa n/2} + \E\left[D'_t-D_t|D_{t-1},D'_{t-1} \in \left[\frac{n}4,\frac{3n}{4}\right]\right] \\
    &\leq 2ne^{- \kappa n/2} + \frac34\E[D'_{t-1}-D_{t-1}].
\end{align*}
For such $t$, $\E[D'_{t-1}-D_{t-1}] \geq 1/n$ and the right hand side is smaller than $(4/5)\E[D'_{t-1}-D_{t-1}]$ for large $n$. By induction we thus get $\E[D'_{t_0+t}-D_{t_0+t}] \leq \max(1/n,(4/5)^tn)$ for $t_0+t \leq L_n$. This gives 
$\E[D'_{t_0+t}-D_{t_0+t}] \leq 1/n$ whenever $L_n \geq t \geq 9\log n$ and $n$ large. Hence $\E[D'_t-D_t]\leq1/n$ for all $L_n \geq t \geq 10\log n$ and $n$ large. By Markov's inequality, $\Pb(D'_t \neq D_t)\leq1/n$ for all $L_n \geq t \geq 10\log n$ and $n$ large. Since if $D'_t=D_t$ for some $t$, then $D'_s=D_s$ for all $s \geq t$, we get $\Pb(D'_t \neq D_t)\leq1/n$ for all $t \geq 10 \log n$.

Thus, taking $T \geq 10\log n$, $\Pb(D_T \neq D'_T) < 1/n$. This means that the total variation distance between the distribution of $D_T$ and $D'_T$ is less than $1/n$, i.e.\ $\sum_{d=0}^n|\Pb(D'_T=d)-\Pb(D_T=d)| < 2/n$. This gives for $T \geq 10\log n+1$
\begin{align*}
    \E[f(\omega_0)f((\omega^{1/2})_0] &=
    \E[h(\omega_T)h((\omega^{1/2})_T)] \\ 
    &= \sum_d \E[h(\omega_T)h((\omega^{1/2})_T)|D'_{T-1}=d]P(D'_{T-1}=d) \\
    &\geq \sum_d \E[h(\omega_T)h((\omega^{\epsilon})_T)|D_{T-1}=d]\Pb(D_{T-1}=d)-\sum_d|\Pb(D'_{T-1}=d)-\Pb(D_{T-1}=d)| \\
    &\geq \E[h(\omega_T)h((\omega^{\epsilon})_T)] - \frac2n \\
    &= \E[f(\omega_0)f((\omega^{\epsilon})_0)] - \frac2n,
\end{align*}
where the first inequality follows from that $|h| \leq 1$.
To now prove quenched QNS for all $h$ and $T \geq 10\log n$ observe that we now have
\begin{align*}
\E[\Cov(f(\omega_0),f((\omega^\epsilon)_0)|\Theta)] &= \E[\Cov(f(\omega_0),f((\omega^\epsilon)_0)|\Theta)-\Cov(f(\omega_0),f((\omega^{1/2})_0)|\Theta)] \\
&=\E[\E[f(\omega_0)f((\omega^\epsilon)_0)|\Theta] - \E[f(\omega_0)f((\omega^{1/2})_0)|\Theta]] \\
&=\E[f(\omega_0)f((\omega^\epsilon)_0)] - \E[f(\omega_0)f((\omega^{1/2})_0)] \\
&<\frac2n.
\end{align*}
Since for any fixed $f$, $\Cov(f(\omega),f(\omega^\epsilon)) \geq 0$, this implies quenched QNS for $T \geq 10\log n$. Combining with (i), this gives (iii).

If $h$ is also odd and $T \geq 10\log n$, we have $h(-(\omega^{1/2})_T)=h((-\omega^{1/2})_T) = -h((\omega^{1/2})_T)$ and since $(\omega,\omega^{1/2}) =_d (\omega,-\omega^{1/2})$, we have $\E[h(\omega_T)h((\omega^{1/2})_T)]=0$. Thus
\[\E[f(\omega_0)f((\omega^\epsilon)_0] \leq \frac{2}{n}.\]

Since $h$ odd implies $\E[f(\omega_0)]=0$, this proves annealed QNS for all odd $h$ and $T \geq 10 \log n$ and combining with (i) we get (ii). 

\smallskip

To prove (iv), we need an example of an $h$ that is not odd and where annealed noise sensitivity does not hold for large $T$. To achieve this, first observe that at each $t$, $\Pb(D_{t} \in \{0,n\}|D_{t-1}) \geq 2^{-n+1}$. Hence, as $D_t \in \{0,n\}$ implies that $D_{t'}=D_{t}$ for all $t' \geq t$, for $T_n$ such that $T_n2^{-n} \rightarrow \infty$, $\Pb(D_{T_n} \not\in \{0,n\}) \rightarrow 0$. This entails $\Pb((\omega^\epsilon)_{T_n} = \pm \omega_{T_n}) \rightarrow 1$. Now let $h_n$ be any even function. Then $\Pb(h_n(\omega_{T_n}) \neq h_n((\omega^\epsilon)_{T_n})) \rightarrow 0$
and so in fact $f_n$ is annealed noise stable. If also $-1<\liminf_n\E[h_n(\omega)] \leq \limsup_n \E[h_n(\omega)] < 1$, then $\{f_n\}$ is not annealed noise sensitive. 

\smallskip
For (v), assume for simplicity that $T_n$ equals the constant $T$; going from this to general bounded $T_n$ is easy and left to the reader. Fix a small $\delta>0$. Since $h_n$ is stable one can pick $\rho>0$ small enough that $\sup_n \Pb(h_n(\omega) \neq h_n(\omega^\xi))< \delta$ whenever $\xi<\rho$. Pick such a $\rho$ and let $\epsilon=\rho^{2^T}$. We have $g(v) = (2/\pi)\arctan\sqrt{v/(1-v)} < (2/3)\sqrt{v}$ for $v<\rho$ and $\rho$ small enough. By Chernoff bounds we then have that there is a $\kappa>0$ independent of $n$ such that
\[\Pb(D^\epsilon_{T-1} \geq n\rho^2) < e^{-\kappa n}.\]
Given $D_{T-1}=nv$ for $v<\rho^2$, the conditional distribution of $(\omega_T,(\omega^\epsilon)_T)$ is that of $(\omega,\omega^{g(v)})$ and since $g(v)<\rho$, we get by noise stability of $h_n$ that
\[\Pb(f_{n,T_n}(\omega) \neq f_{n,T_n}(\omega^\epsilon)|D^\epsilon_{T-1} < \rho^2) < \delta.\]
Summing up, this now gives
\[\Pb(f_{n,T_n}(\omega) \neq f_{n,T_n}(\omega^\epsilon)) < e^{-\kappa n}+\delta\]
This easily implies noise stability.

\end{proof}

\begin{proof}[Proof of Theorem \ref{th:stronger_noise_sensitivity_deep_networks}]
    Starting with (i), this follows from the simple observation that if $d$ is the number of disagreements between $\omega$ and $\eta$, then $\omega_1$ is a vector i.i.d.\ fair coin flips and given $\omega_1$, $\eta_1$ differs from $\omega_1$ in uniformly random positions whose number is binomial with parameters $n$ and $g(d/n)$. Hence all arguments from (i) or (ii) in Theorem \ref{th:noise_sensitivity_deep_networks} can be copied from above from $t=1$.

Part (ii) is an immediate corollary of (i) on observing that \[\Cov_{\mathbf{Q_n},\Theta}(f_{n,T_n}(\omega),f_{n,T_n}(\eta)) = \E_{\mathbf{Q}_n}[\Cov_{\Theta}(f_{n,T_n}(\omega),f_{n,T_n}(\eta))] + \Cov_{\mathbf{Q}_n}(\E_\Theta[f_{n,T_n}(\omega)],\E_\Theta[f_{n,T_n}(\eta)])\]
and by the invariance and symmetry properties of $\Theta$, $\E_\Theta[f_{n,T_n}(\omega)]$ is independent of $\omega$ and thus a constant and hence $\Cov_{\mathbf{Q}_n}(\E_\Theta[f_{n,T_n}(\omega)],\E_\Theta[f_{n,T_n}(\eta)]) = 0$.

\smallskip 

Now for (iii) fix $\omega$, let $A_k$ be the set of $\eta'$ that each differs from $\omega$ at $k$ positions and take $\eta \in A_k$. We have by (i), since $h_n$ odd implies $\E[f_{n,T_n}(\omega)]=\E[f_{n,T_n}(\eta)]=0$, that $|\E[f_{n,T_n}(\omega)f_{n,T_n}(\eta)]| < \delta^4$, i.e.\ $\Pb(f_{n,T_n}(\omega) \neq f_{n,T_n}(\eta)) \in (1/2-1/2\delta^4,1/2+1/2\delta^4)$, for $n$ large.

Now let $\bar{M}_k=\bar{M}^{(n)}_k(\omega)=|\{(\eta,\xi) \in A_k:f_{n,T_n}(\eta) \neq f_{n,T_n}(\xi)\}|$. It follows that
\[\E[\bar{M}_k] \in \left((1-\delta^4)\frac12\binom{\binom{n}{k}}{2},(1+\delta^4)\frac12\binom{\binom{n}{k}}{2}\right) \subseteq \left((1-\delta^4)\frac{1}{4}\binom{n}{k}^2,(1+\delta^4)\frac{1}{4}\binom{n}{k}^2\right).\]
Let $M$ be the maximum value that $\bar{M}_k$ can take on, so that
\[M \leq \left(\frac{\binom{n}{k}}{2}\right)^2 = \frac14 \binom{n}{k}^2.\]
Since $M-\bar{M}_k$ is nonnegative, it follows from Markov's inequality that
\[\Pb\left(\bar{M}_k \leq (1-\delta^2)\frac{1}{4}\binom{n}{k}^2\right) < \delta^2.\]
Also if $\bar{M}_k > (1-\delta^2)\frac{1}{4}\binom{n}{k}^2$, we have $X_k^+ \in ((1-\delta)\frac{1}{2}\binom{n}{k},(1+\delta)\frac{1}{2}\binom{n}{k})$, where $X_k^+$ is the number of $\eta \in A_k$ with $f_{n,T_n}(\eta)=1$. This gives $M_k \in ((1-\delta)\frac{1}{2}\binom{n}{k},(1+\delta)\frac{1}{2}\binom{n}{k})$ and hence for $n$ sufficiently large,
\[\Pb\left(\frac{M_k}{\binom{n}{k}} \not\in \left(\frac{1-\delta}{2},\frac{1+\delta}{2}\right)\right) < \delta^2.\]
This proves (iii).

\smallskip 
\end{proof}

(The property described in the statement (ii) in Theorem \ref{thm:stronger_correlated_noise_sensitivity} could be referred to as $f_{n,T_n}$ being annealed noise sensitive with respect to $\mathbf{Q}_n$.)

\subsection{Correlated networks} \label{sec:correlated}
    
    In this section we will investigate the noise sensitivity of a deep network where the network weights $\theta$ are sampled from a correlated normal distribution, i.e.\ Case 2 in the introduction.
    
    More precisely, the model is that the random matrices $\theta_1,\theta_2,\ldots,\theta_T$ are independent and for given $t$, the columns $\theta_t(\cdot,1),\ldots,\theta_t(\cdot,n)$ are independent. However, each column $\theta_t(\cdot,j)$ of each $\theta_t$ is now assumed to be $n$-dimensional Gaussian with expectation $0$ and covariance matrix $\Sigma$, where $\Sigma_{i,i}=1$, $i=1,\ldots,n$ and $\Sigma_{i,i'}=\rho$, $1 \leq i < i' \leq n$. Here $\rho=\rho_n$ is a positive given correlation and we are interested in providing conditions on $\rho$ that ensure noise stability or noise sensitivity of $f=f_n=f_{n,T_n}$ defined as in the previous section.
    
    \smallskip 


\smallskip 

To model the vectors $\theta_t(\cdot,j)$, we let $\theta_t(i,j)=\sqrt{\rho}\nu_t(j)+\sqrt{1-\rho}\psi_t(i,j)$, where the $\nu_t(j)$'s and $\psi_t(i,j)$'s are all independent standard Gaussian.

\smallskip

Since the entries in any given row of $\theta_t$ are i.i.d.\ standard normals, Lemma \ref{th:tan_prop} still says that $\Pb(F_t(i)|D_{t-1}=d) = g(d/n)$, where $F_t(i)=\{\omega_t(i) \neq (\omega^\epsilon)_t(i)\}$. However the events $F_t(i)$ and $F_t(i')$ are not, as in Section 2, conditionally independent given $D_{t-1}$.

\smallskip 

Recall the proof of Lemma \ref{th:tan_prop}, where the following observation was made. Here we recall that $C=\{j:\omega_{t-1}(j) \neq (\omega^\epsilon)_{t-1}(j)\}$, $A_{\Lambda,k} = \{j : j \in \Lambda, \omega(j) = k\}$ and
\[F_t(i) = \left\{\left|\sum_{j \in A_{C^c,1}} \theta(i,j) - \sum_{j \in A_{C^c,-1}} \theta(i,j)\right| \leq \left|\sum_{j \in A_{C,1}} \theta(i,j) - \sum_{j \in A_{C,-1}} \theta(i,j)\right|\right\}.\]
Using the above representations of $\theta_t(i,j)$ this becomes
\begin{align*}
    F_t(i) &= \left\{ \left|\sum_{j \in A_{C^c,1}} \left(\sqrt{\rho}\nu_t(j)+\sqrt{1-\rho}\psi_t(i,j)\right) - \sum_{j \in A_{C^c,-1}} \left(\sqrt{\rho}\nu_t(j)+\sqrt{1-\rho}\psi_t(i,j)\right)\right| \right. \\ &\leq \left. \left|\sum_{j \in A_{C,1}} \left(\sqrt{\rho}\nu_t(j)+\sqrt{1-\rho}\psi_t(i,j) \right) - \sum_{j \in A_{C,-1}} \left(\sqrt{\rho}\nu_t(j)+\sqrt{1-\rho}\psi_t(i,j) \right)\right|\right\}.  
\end{align*}
Make the following substitutions

\begin{align*}
        \sqrt{d}U_t^{C} &= \sum_{j \in A_{C,1}} \nu_t(j) - \sum_{j \in A_{C,-1}} \nu_t(j) \\ \sqrt{n-d}U_t^{C^c} &= \sum_{j \in A_{C^c,1}} \nu_t(j) - \sum_{j \in A_{C^c,-1}} \nu_t(j) \\
        \sqrt{d}V_t^{C}(i) &= \sum_{j \in A_{C,1}} \psi_t(i,j) - \sum_{j \in A_{C,-1}} \psi_t(i,j) \\ \sqrt{n-d}V_t^{C^c}(i) &= \sum_{j \in A_{C^c,1}} \psi_t(i,j) - \sum_{j \in A_{C^c,-1}} \psi_t(i,j)
\end{align*}

and notice that $U_t^C$, $U_t^{C^c}$, $V_t^{C}(i)$ and $V_t^{C^c}(i)$ are all independent standard Gaussians for all $t$ and $i$. This gives, with $v=d/n$,

\[F_t(i) = \left\{\left| \sqrt{\frac{\rho}{1-\rho}}U^{C^c}_t + V^{C^c}_t(i)\right| \leq \sqrt{\frac{v}{1-v}}\left| \sqrt{\frac{\rho}{1-\rho}}U^C_t + V^C_t(i) \right|\right\}\]

and hence also

\[\Pb(F_t(i)|\omega_{t-1},(\omega^\epsilon)_{t-1}) 
= \Pb\left( \left| \sqrt{\frac{\rho}{1-\rho}}U^{C^c}_t + V^{C^c}_t(i)\right| \leq \sqrt{\frac{v}{1-v}}\left| \sqrt{\frac{\rho}{1-\rho}}U^C_t + V^C_t(i) \right|\right).\]
The dependence between different $i$'s is captured by the common variables $U^C_t$ and $U^{C^c}_t$. Conditioning on $W_t=\sqrt{\rho/(1-\rho)}(U^C_t,U^{C^c}_t)$ in addition to the condition $D_{t-1}=d$, one gets that $F_t(1),\ldots,F_t(n)$ are conditionally independent and, writing $w=(w^C,w^{C^c})$,
\begin{equation} \label{eq:g_W}
\Pb(F_t(i)|D_{t-1}=d,W_t=w) =
\Pb\left(|Y| \leq \sqrt{\frac{v}{1-v}}|X|\right),
\end{equation}
where $X$ and $Y$ are independent normals with unit variance and means $w^{C}$ and $w^{C^c}$ respectively. Write $g_w(v)$ for the right hand side of (\ref{eq:g_W}). Summing up, by letting $\tilde{D}_t = D_t/n$, we have shown that the conditional distribution
$(D_t|\tilde{D}_{t-1}=v,W=w)$ is binomial with parameters $n$ and $g_w(v)$.

Indeed given $\tilde{D}_{t-1} = v$, $D_t$ can be determined in two steps. First sample $W$ from the two dimensional independent Gaussian with covariance $\sqrt{\rho/(1-\rho)} I_2$. Then given $W = w$, $D_t$ is determined as the number of independent two-dimensional Gaussian's $(X_i,Y_i)$ with mean $w$ and $X_i$ and $Y_i$ of unit variance and independent that end up in the region $A_t = \{(x,y) \: : \: |y| \leq \sqrt{\frac{v}{1-v}}|x|\}$. The probability of a given sample in the second step ending up in $A_t$ is $g_w(v)$. From this it is easy to see that $g_w(v)$ is increasing in $v$ for all values of $w$. An illustration of $A_t$ and $W$ is seen in figure \ref{fig:At}.
\smallskip


\medskip 

\begin{figure}[h] 
    \centering
    \begin{subfigure}[b]{0.4\linewidth}
        \includegraphics[width = 1\textwidth]{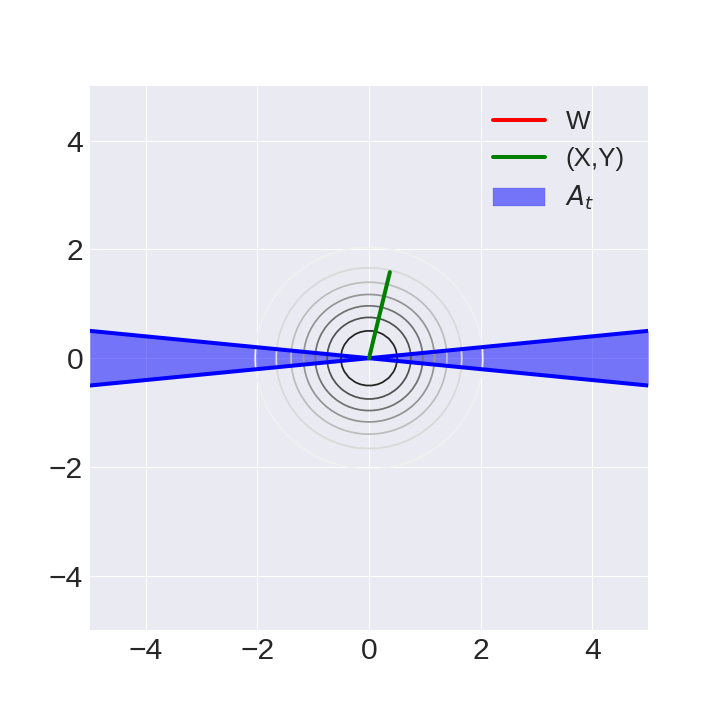}
        \caption{$W = (0,0)$. Corresponds to the uncorrelated case.}
        \end{subfigure}
    \begin{subfigure}[b]{0.4\linewidth}
        \includegraphics[width = 1\textwidth]{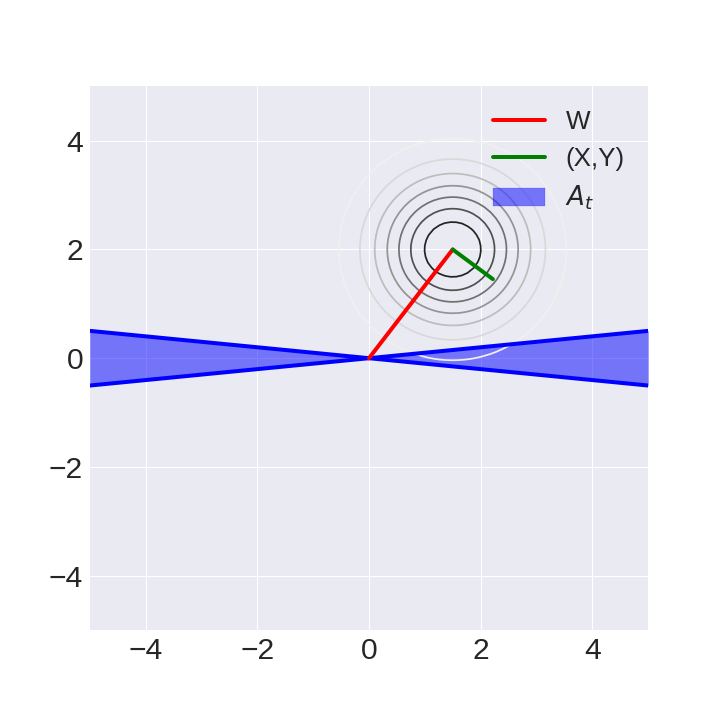}
        \caption{$W = (1.5,2)$. One realisation of the correlated case.}
        \label{fig:correlated_case}
        \end{subfigure}
    \caption{An illustration of how $D_t$ is sampled given $\tilde{D}_{t-1} = v$. First sample $W$ from a independent two dimensional Gaussian with covariance matrix $\sqrt{\rho/(1- \rho)} I_2$. Then given $W = w$, $D_t$ is determined as the number of independent Gaussian's with mean $w$ and unit variance that end up in the region $A_v$.}
    \label{fig:At}
\end{figure}

We will need the following definition.

\begin{definition}
    A sequence of Boolean function $\{h_n\}$ has a { \bf sharp threshold at 1/2} if $\{h_n\}$ is such that for $\omega$ being i.i.d. Bernoulli($p_n$) and if $\{p_n - 1/2\}$ is bounded away from $0$, then $\lim_{n \rightarrow \infty} \Pb(h_n(\omega)= \sign(p_n - 1/2)) = 1$.
\end{definition}

One among many examples of Boolean functions that has a sharp threshold at 1/2 is the majority function. It can be shown that this also is true for the weighted majority function when the weights $\theta_i > 0$ are such that $\lim_{n \rightarrow \infty} \frac{\max_i \theta_i}{\sqrt{n}\min_i \theta_i} = 0$.

Let $\tau_c^\epsilon = \tau_c^{\epsilon_n}=\min\{t:g_{W_t}(\tilde{D}_{t-1}) \geq c\}$ with $\tau_c^\epsilon=\infty$ if $g_{W_t}(\tilde{D}_{t-1})<c$ for all $t$.
The following lemma is crucial. It shows that 
noise sensitivity or noise stability almost entirely come down to if $\tilde{D}$ started from arbitrary small $\epsilon$ hits $1/2$ before $0$ (which gives rise to sensitivity results) or vice versa (giving rise to stability results). In both of these cases, the rest of the arguments come down to adopting suitable conditions on $h_n$ to go along with that.
Therefore and since we are also of the opinion that $\tilde{D}$ is a very natural Markov process that is interesting in its own right, we will later on state the results for $\tilde{D}$ explicitly along with sensitivity/stability results.

\begin{lemma} \label{lem:correlated_sens_stab}
    Let $1/2 \geq \epsilon_n \downarrow 0$. Assume $\{h_n\}$ has a sharp threshold at 1/2, $\rho_n > \delta$ for some $\delta > 0$. Then the following statements hold.
    \begin{enumerate}
        \item[(i)] If for all $\epsilon_n$, $\lim_{n \rightarrow \infty} \Pb(\tau_{1/2}^{\epsilon_n} \geq T_n) = 0$ and $\{h_n\}$ is odd, then $\{f_{n,T_n}\}$ is annealed, and hence quenched, QNS at level $\{\epsilon_n\}$.
        \item[(ii)] If for all $\epsilon>0$ sufficiently small, $\lim_{n \rightarrow \infty} \Pb(\tau_{2\epsilon}^\epsilon \leq T_n) = 0$, then $\{f_{n,T_n}\}$ is annealed, and hence quenched, noise stable.
    \end{enumerate}
    
\end{lemma}
\begin{proof}
    Let $T = T_n$ and $f = f_{n,T}$.
    As before, write $\theta_t(i,j)=\sqrt{\rho}\nu_t(j)+\sqrt{1-\rho}\psi_t(i,j)$, which in row vector form becomes $\theta_t(i,\cdot)=\sqrt{\rho}\nu_t+\sqrt{1-\rho}\psi_t(i,\cdot)$, so that 
    \[\omega_T(i) = \sign \left((\sqrt{\rho}\nu_T+\sqrt{1-\rho}\psi_T(i,\cdot)) \cdot \omega_{T-1}\right).\]
    
    This means that conditionally on $\omega_{T-1}$ and $\nu_T$, $\omega_T(i)$ are i.i.d. Bernoulli with probability $p(\nu_T,\omega_{T-1}) \\ := \Pb(\omega_T(i)=1|\nu_T,\omega_{T-1}) = \Pb((\sqrt{\rho}\nu_T+\sqrt{1-\rho}\psi_T(i,\cdot)) \cdot \omega_{T-1} >0)$, which equals
    \[\Pb \left( X > -\sqrt{\frac{\rho}{1-\rho}}\frac{1}{\sqrt{n}}\nu_T \cdot \omega_{T-1}\right)\]

    where $X$ is a standard normal random variable. Notice that $p(\nu_T,\omega_{T-1}) > \frac{1}{2}$ if and only if $\nu_T \cdot \omega_{T-1} > 0$. Fix $\gamma_1 > 0$. Since $\rho_n > \delta$ for some $\delta$, there is for every $\gamma_1 > 0$ an $\gamma_2>0$ independently of $\rho$ and $\omega_{T-1}$ such that for $n$ large

\begin{equation} \label{eq:bounded_from_half}
  \Pb \left( p(\nu_T,\omega_{T-1}) \in \left[\frac{1}{2} - \gamma_2, \frac{1}{2} + \gamma_2\right] \Big| \, \omega_{T-1}\right) < \gamma_1.  
\end{equation}
    
  Let us now also include the process $(\omega^\epsilon)_t = (\omega^{\epsilon_n})_t$. By (\ref{eq:g_W}), $\Pb(F_t(i)|\tilde{D}_{t-1} = v)$ is increasing in $v$, which means that the conditional distribution $(\tilde{D}_{t}|\tilde{D}_{t-1} =v)$ is stochastically increasing in $v$. Hence inductively for $t \leq T-1$, $(\tilde{D}_{T-1}|\tilde{D}_{t-1} =v)$ is stochastically increasing in $v$.
  Additionally, as in the proof of Lemma \ref{th:tan_prop}, writing $B_T$ for the event $\{\sign(\nu_T \cdot \omega_{T-1}) \neq \sign(\nu_T \cdot (\omega^\epsilon)_{T-1}\}$,
  
  \begin{equation} \label{eq:sign_nu}
      \Pb \left(B_T|\omega_{T-1},(\omega^\epsilon)_{T-1}\right) = g(\tilde{D}_{T-1})
  \end{equation}
  from which it follows by symmetry that for all $\eta \in [0,1/2]$
  \[\Pb\left(B_T|\tilde{D}_{t}=\frac12-\eta\right) + \Pb\left(B_T|\tilde{D}_{t}=\frac12+\eta\right) = 1.\]
  It follows directly that if $n$ is even $\Pb(B_T|\Tilde{D}_t=1/2) = 1/2$ and if $n$ is odd, since $\Pb(B_T|\tilde{D}_{t}=1/2-\eta) < \Pb(B_T|\tilde{D}_{t}=1/2+\eta)$, that $\Pb(B_T|\tilde{D}_t=1/2 + 1/2n)\geq 1/2$. Assume for simplicity for the rest of the proof of (i) that $n$ is even; for the case with $n$ odd just replace any conditioning on $\Tilde{D}_t=1/2$ with conditioning on $\Tilde{D}_t=1/2+1/2n$.
  
  Since $(\tilde{D}_{T-1}|\tilde{D}_{t-1} =v)$ is increasing in $v$, we now get
  \[\Pb(B_T|\tau_{1/2}^\epsilon=t) \geq \Pb\left(B_T|\tilde{D}_t=\frac12\right) = \frac12.\]
  Now let $A_T= B_T \setminus C_T$, where 
  \[C_T=\{p(\nu_T,\omega_{T-1}) \in [1/2-\gamma_2,1/2+\gamma_2]\} \cup \{p(\nu_T,(\omega^\epsilon)_{T-1}) \in [1/2-\gamma_2,1/2+\gamma_2]\}.\]
  By (\ref{eq:bounded_from_half}) and the fact that $A_T \subset B_T$, it now follows for $t < T$
    
    \[\Pb\left(A_T|\tau_{1/2}^\epsilon=t\right) \geq \frac12-\Pb(C_T|\tau_{1/2}^\epsilon=t) > \frac12-2\gamma_1.\]
    
    By the assumptions on $\{h_n\}$, for sufficiently large $n$ and $t < T$, $\E[f(\omega)f(\omega^\epsilon)|A_T,\tau_{1/2}^\epsilon=t] < -1 + \gamma_1$.
    We then get
    
    \begin{align*}
    \E\left[f(\omega)f(\omega^\epsilon)|\tau_{1/2}^\epsilon = t\right]
    &= \E\left[f(\omega)f(\omega^\epsilon)|A_T,\tau_{1/2}^\epsilon=t\right]\Pb\left(A_T|\tau_{1/2}^\epsilon=t\right) \\
    &+
    \E\left[f(\omega)f(\omega^\epsilon)|A^c_T,\tau_{1/2}^\epsilon=t\right]\Pb\left(A^c_T|\tau_{1/2}^\epsilon=t\right) \\ &< (-1 + \gamma_1)\left(\frac{1}{2} - 2\gamma_1\right) + \frac{1}{2} + 2\gamma_1 < 5\gamma_1.
    \end{align*}
    
    Finally this means that for $n$ sufficiently large,
    
    \begin{align*}
    0 &\leq \E\left[f(\omega)f(\omega^\epsilon)\right] \\
    &= \sum_{t=0}^{T-1}\E\left[f(\omega)f(\omega^\epsilon)|\tau_{1/2}^\epsilon =t\right]\Pb\left(\tau_{1/2}^\epsilon =t\right) + \E\left[f(\omega)f(\omega^\epsilon)|\tau_{1/2}^\epsilon \geq T\right]\Pb\left(\tau_{1/2}^\epsilon \geq T\right) \\
    &< 5\gamma_1 + \Pb\left(\tau_{1/2}^\epsilon \geq T\right).
    \end{align*}
    
    Since $h$ is assumed odd, so is $f$ and hence $\Cov(f(\omega),f(\omega^\epsilon)) = \E[f(\omega)f(\omega^\epsilon)]$. Since $\gamma_1$ is arbitrary, (i) follows.

    \smallskip

    To prove (ii), fix $\gamma_1$, pick $\epsilon$ such that $g(2\epsilon)<\gamma_1$ and pick $\gamma_2$ such that (\ref{eq:bounded_from_half}) holds. Notice that due to (\ref{eq:sign_nu}) and the fact that $g$ is increasing, we have
    
    \[\Pb\left(B_T|\tilde{D}_{T-1} \leq 2\epsilon\right) \leq g(2\epsilon) < \gamma_1.\]

    Hence by (\ref{eq:bounded_from_half})
    \[\Pb(B_T \cup C_T|\tilde{D}_{T-1} \leq 2\epsilon) < 3\gamma_1.\]
    
    Additionally for $n$ large and $\epsilon$ sufficiently small, due to $h_n$ having a sharp threshold at 1/2, 
    \[\Pb\left(f(\omega) \neq f(\omega^\epsilon)|B_T^c \cap C_T^c,\tilde{D}_{T-1} \leq 2\epsilon\right)  < \gamma_1.\] This results in
    
    \begin{align*}
    &\Pb\left(f(\omega) \neq f(\omega^\epsilon)|\tilde{D}_{T-1} \leq 2\epsilon\right)\\ &= \Pb\left(f(\omega) \neq f(\omega^\epsilon)|B_{T} \cup C_T,\tilde{D}_{T-1} \leq 2\epsilon\right)\Pb\left(B_T \cup C_T|\tilde{D}_{T-1} \leq 2\epsilon\right)\\ &+ \Pb\left(f(\omega) \neq f(\omega^\epsilon)|B_{T}^c \cap C_T^c,\tilde{D}_{T-1} \leq 2\epsilon\right)\Pb\left(B_{T}^c \cap C_T^c|\tilde{D}_{{T}-1} \leq 2\epsilon\right)\\&< 3\gamma_1 + \Pb\left(f(\omega) \neq f(\omega^\epsilon)|B_{T}^c \cap C_T^c,\tilde{D}_{T-1} \leq 2\epsilon\right) < 4\gamma_1.
    \end{align*}

    Since $\lim_{n \rightarrow \infty} \Pb(\tau_{2\epsilon}^\epsilon \leq T) = 0$, we have for $n$ sufficiently large that $\Pb(\tilde{D}_{T-1} > 2\epsilon) \leq \Pb(\tau_{2\epsilon}^\epsilon \leq T) < \gamma_1$.

    It now follows that for $n$ large,
    
    \begin{align*}
    &\Pb\left(f(\omega) \neq f(\omega^\epsilon)\right)\\ &< \Pb\left(f(\omega) \neq f(\omega^\epsilon)|\tilde{D}_{T-1} \leq 2\epsilon\right) + \Pb\left(\tilde{D}_{T-1} > 2\epsilon\right) \\&<5 \gamma_1
    \end{align*}

    Since $\gamma_1$ is arbitrary, this concludes (ii).
    
\end{proof}

We are now ready to start the proof that $\tilde{D}$ hits $1/2$ before $0$ and $f=f_{n,T_n}$ is annealed (and thus quenched) QNS for a large range of $\rho$ 
if $T_n$ is suitably large. The results are stated in Theorem \ref{thm:correlated_noise_sensitivity}. 

\medskip 

We are going to make several observations concerning the behaviour of $g_W(v)$. Since we will for the most parts only be considering a single $t$, we will in such cases drop $t$ from the notation.

Before going on, recall that the sum of the squares of two independent standard normal random variables is exponential with mean $2$. This also means that $r^2:=||W||_2^2$ is exponential with mean $2\rho/(1-\rho)$. 

\begin{lemma} \label{lem:lower_bound_1_g_w}
Let $r=||w||_2$. Then
\[g_w(v) \geq \frac{2}{\pi}\arctan \left(\sqrt{\frac{v}{1-v}} \right)e^{-r^2/2} = g(v)e^{-r^2/2}.\]

\end{lemma}
\begin{proof}
        
        Let $w = (a,b)$. Then
        
        \begin{equation*}
        g_w(v) = \frac{1}{2\pi} \int_{-\infty}^{\infty} \int_{-\sqrt{\frac{v}{1-v}} |x|}^{\sqrt{\frac{v}{1-v}}|x|} e^{-((x-a)^2 + (y-b)^2 )/2}dy dx = \frac{e^{-\frac{a^2 + b^2}{2}}}{2\pi} \int_{-\infty}^{\infty} \int_{-\sqrt{\frac{v}{1-v}} |x|}^{\sqrt{\frac{v}{1-v}}|x|} e^{-(x^2 + y^2)/2}e^{xa + yb} dydx
    \end{equation*}
    
    \begin{equation*}
        \geq\frac{e^{-\frac{r^2}{2}}}{2\pi} \int_{-\infty}^{\infty} \int_{-\sqrt{\frac{v}{1-v}} |x|}^{\sqrt{\frac{v}{1-v}}|x|} e^{-(x^2 + y^2)/2} dydx = \frac{2}{\pi} \arctan\left( \sqrt{\frac{v}{1-v}}\right) e^{-\frac{r^2}{2}}
    \end{equation*}
     where the inequality follows from symmetry of the integrated function around the origin and that $e^{x} + e^{-x} \geq 2$. The last equality follows from the computation in the proof of Lemma \ref{th:tan_prop}.
\end{proof} 

\begin{lemma} \label{lem:lower_bound_2_g_w}
Assume that $v\leq 1/2$ and that $w=(w^C,w^{C^c})$ satisfies $|w^{C^c}| \leq \sqrt{v/(1-v)}|w^C|-2$. Then $g_w(v)>1/2$.
\end{lemma}

\begin{proof}

Let $X \sim N(w^C,1)$ and $Y \sim N(w^{C^c},1)$ be independent and write $X=w^C+\xi$ and $Y=w^{C^c}+\eta$ for independent standard normal $\xi$ and $\eta$. We have
\[g_w(v) = \Pb\left(|Y| \leq \sqrt{\frac{v}{1-v}}|X|\right) = \Pb((X,Y) \in A_v)\]
where 
\[A_v = \left\{(x,y):|y| \leq \sqrt{\frac{v}{1-v}}|x|\right\}.\]
Since $v \leq 1/2$, it is easily seen that the $L_2$ distance between $(w^C,w^{C^c})$ and $A_v^c$ is smaller or equal to $\sqrt{2}$
whenever $|w^{C^c}| \leq \sqrt{v/(1-v)}|w^C|-2$. Hence
\[g_w(v) \geq \Pb(\xi^2+\eta^2 \leq 2) = 1-e^{-1} > \frac12.\]
\end{proof}

\begin{lemma} \label{lem:lower_bound_3_g_w}
Let as above $W=W_t=(W^C,W^{C^c})$ for $t < T$, where $W^C$ and $W^{C^c}$ are independent normals with means 0 and variance $\rho/(1-\rho)$. Then for $v\leq1/2$,
\[\Pb\left(g_W(v)>\frac12\right) > \frac{1}{2\pi}\sqrt{\frac{v}{1-v}}e^{-\frac{8(1-\rho)}{\rho}\frac{1-v}{v}}.\]
\end{lemma}

\begin{proof}
By Lemma \ref{lem:lower_bound_2_g_w}, $\Pb(g_W(v)>1/2) \geq \Pb(|W^{C^c}| \leq \sqrt{v/(1-v)}|W^C|-2)$. Writing $(r,\varphi)$ for the polar coordinates of $W$, it is straightforward by back substitution to see that it is sufficient for
$|W^{C^c}| \leq \sqrt{v/(1-v)}|W^C|-2$ that 
\[r^2>16(1-v)/v\] 
and 
\[|\varphi|<\arctan\left(\frac12\sqrt{\frac{v}{1-v}}\right)\]
and the latter in turn occurs whenever $|\varphi|<\frac14\sqrt{v/(1-v)}$. 
Since $r^2$ and $|\varphi|$ are independent and exponential($(1-\rho)/2\rho$) and uniform on $[0,\pi/2]$ respectively, we thus get
\[\Pb\left(g_W(v)>\frac12\right) > e^{-\frac{8(1-\rho)}{\rho}\frac{1-v}{v}}\frac{2}{\pi}\frac14 \sqrt{\frac{v}{1-v}}\]
which is the desired bound.
\end{proof}

We are now ready to state the first main results of this section.

\begin{theorem} \label{thm:correlated_noise_sensitivity}
Let $\{\epsilon_n\}$ be such that $1/2 \geq \epsilon_n \downarrow 0$ and $n\epsilon_n \rightarrow \infty$. Assume that for some $\delta>0$ independent of $n$, $\delta< \rho < 1-\frac{4(\log\log n)^3}{\log n}$ and $T_n \geq e^{4(\log \log n)^2}$. Then the following statements hold.
\begin{itemize}
    \item[(i)] 
    $\lim_{n \rightarrow \infty}\Pb(\tau^{\epsilon_n}_{1/2}>T_n)=0$,
    \item[(ii)] If $h_n$ is odd and has a sharp threshold at 1/2, then $\{f_{n,T_n}\}$ is annealed QNS and hence also quenched QNS at level $\{\epsilon_n\}$.
    \item[(iii)] If $h_n$ is odd and $T_n$ grows sufficiently large with $n$,  then $\{f_{n,T_n}\}$ is annealed QNS and hence also quenched QNS at level $\{\epsilon_n\}$.
\end{itemize} 

\end{theorem}

As in the uncorrelated case, we can prove even stronger versions.

\begin{theorem} \label{thm:stronger_correlated_noise_sensitivity}
    Let $\{\epsilon_n\}$ be such that $1/2 \geq \epsilon_n \downarrow 0$ and $n\epsilon_n \rightarrow \infty$. Assume that for some $\delta>0$ independent of $n$, $\delta< \rho < 1-\frac{4(\log\log n)^3}{\log n}$ and $T_n \geq e^{4(\log \log n)^2}$. Then the following statements hold.
    \begin{itemize}
    \item[(i)] Fix $\omega,\eta \in \{-1,1\}^n$ arbitrarily such that $\eta \not\in \{\omega,-\omega\}$. Then if $h_n$ is odd and either has a sharp threshold at 1/2 or $T_n$ is sufficiently large, then
    \[\lim_{n \rightarrow \infty} \Cov_\Theta(f_{n,T_n}(\omega),f_{n,T_n}(\eta)) = 0.\]
    \item[(ii)] Let $\mathbf{Q}_n$ be any probability measure on $\{-1,1\}^n \times \{-1,1\}^n$ such that $\lim_{n \rightarrow \infty} \mathbf{Q_n}(\eta \in \{\omega,-\omega\})=0$, Then if $h_n$ is odd and either has a sharp threshold at 1/2 or $T_n$ is sufficiently large, then
        \[\lim_{n \rightarrow \infty} \Cov_{\mathbf{Q_n},\Theta}(f_{n,T_n}(\omega),f_{n,T_n}(\eta)) = 0.\]
    \item[(iii)] Assume that $h_n$ is odd and either has a sharp threshold at 1/2 or $T_n$ is sufficiently large. Fix any $k \in \{1,2,\ldots,n-1\}$ and $\delta>0$. Fix also $\omega \in \{-1,1\}^n$ and let $M_k=M^{(n)}_k(\omega)$ be the number of $\eta$ with $\eta(i) \neq \omega(i)$ for exactly $k$ indexes $i$, such that $f_{n,T_n}(\eta) \neq f_{n,T_n}(\omega)$. Then for $T_n \geq K_n$,
        \[\lim_{n \rightarrow \infty}\Pb\left(\frac{M_k}{\binom{n}{k}} \not \in \left(\frac{1-\delta}{2},\frac{1+\delta}{2}\right)\right) = 0.\]
\end{itemize}
\end{theorem}

\begin{proof}[Proof of Theorem \ref{thm:correlated_noise_sensitivity}]
Let us first outline the strategy of the proof of (i). We will run the $\Tilde{D}$-process for a predetermined time $s$. In doing so, we will prove that (a) regardless of $\epsilon \in (0,1/2]$, the probability that $\Tilde{D}$ hits $0$ during time $s$ is very small regardless of the value of $\Tilde{D}_0$ as long as it is nonzero, and (b) the probability that $\Tilde{D}$ hits $1/2$ during time $s$ is of much higher order. Having done that allows us to repeatedly run the process for $s$ units of time and use the Markov property to draw the conclusion that with very high probability, $\Tilde{D}$ hits $1/2$ before $0$ for sufficiently fast growing $T_n$. The final part (c) upper bounds the time it takes to hit either $1/2$ or $0$.

To avoid confusion we point out that it will be easy for the reader to see that many of the bounds given are far from optimal and thus to some extent arbitrary; they are simply good enough for their purpose. 

\smallskip 

Let $s= s_n = \log_2\log n$ and fix $\epsilon > 0$. We start by giving an upper bound on $P(\tilde{D}_s=0)$. 
We have for large $n$ and all $t < T_n$,
\begin{align*}
    \Pb(\tilde{D}_{t+1}=0 | \tilde{D}_t>0) &\leq \Pb\left(\tilde{D}_{t+1}=0 \Big| \tilde{D}_t=\frac1n\right) \\
    &\leq \Pb\left(\frac{||W||_2^2}{2} > \frac{2\log n}{\log\log n}\right) + \left(1-g\left(\frac1n\right)e^{-\frac{2\log n}{\log\log n}}\right)^n \\
    &< e^{-\frac{4(\log \log n)^3}{\log n}\frac{2\log n}{\log \log n}} + \left(1-\frac{1}{n^{2/3}}\right)^n \\
    &<e^{-8(\log\log n)^2} + e^{-n^{1/3}} \\
    &<e^{-7(\log\log n)^2},
\end{align*}
where the second inequality uses Lemma \ref{lem:lower_bound_1_g_w} for the second term. By the Markov property of $\tilde{D}$ and Bonferroni, we get
\begin{equation} \label{eq:upper_bound_on_dying_out} 
\Pb(\tilde{D}_s=0) < (\log_2\log n) e^{-7(\log\log n)^2} < e^{-6(\log\log n)^2}.
\end{equation}
This concludes part (a) in the sketch.

Next, we  lower bound $\Pb(g_{W_{s}}(D_{s-1})>1/2)$. There is $\kappa>0$ such that for any $v \in [1/n,1/2)$,
\begin{align*}
    \Pb\left(\tilde{D}_{t+1} \geq \frac{\sqrt{v}}{20} \Big| \tilde{D}_t=v\right) &\geq \Pb\left(g_W(v) > \frac{\sqrt{v}}{20}\right)\, 
    \Pb\left(\tilde{D}_{t+1} \geq \frac{\sqrt{v}}{20} \Big| g_W(v) > \frac{\sqrt{v}}{10}\right)\nonumber \\
    &> \Pb\left(\frac{||W||_2^2}{2} \leq 1\right)\left(1-e^{-\kappa \sqrt{n}}\right) \nonumber \\
    &=\left(1-e^{-\frac{1-\rho}{\rho}}\right)\left(1-e^{-\kappa \sqrt{n}}\right) \nonumber \\
    &> 1-e^{-\frac12\frac{4(\log\log n)^3}{\log n}} \nonumber \\
    &> \frac{(\log \log n)^3}{\log n}.
    \end{align*}
Here the second inequality is Lemma \ref{lem:lower_bound_1_g_w} and Chernoff bounds. This gives, provided that $\tilde{D}_0 \geq 1/n$,
\begin{equation} \label{eq:lower_bound_on_probability_of_growth_of_D}
\Pb\left(\forall t=0,\ldots,s-2:\tilde{D}_{t+1} \geq \min\left(\frac12,\frac{\sqrt{\tilde{D}_{t}}}{20}\right)\right) > \left(\frac{(\log\log n)^3}{\log n}\right)^{\log_2\log n} > e^{-2(\log \log n)^2}.
\end{equation}

If the event in the left hand side occurs, then
\[\tilde{D}_{s-1} \geq \frac{1}{400}\left(\frac1n\right)^{2^{-\log_2\log n}} = \frac{1}{400}e^{-1} > \frac{1}{1200}.\]
By Lemma \ref{lem:lower_bound_3_g_w},
\[\Pb\left(g_{W_{s}}(\tilde{D}_{s-1}) > \frac12 \Big| \tilde{D}_{s-1}> \frac{1}{1200}\right) > \frac{1}{40\pi}e^{-\frac{9600(1-\rho)}{\rho}} \geq \frac{1}{40\pi}e^{-\frac{9600(1-\delta)}{\delta}} =: a.\]
Combining with (\ref{eq:lower_bound_on_probability_of_growth_of_D}), it follows that
\begin{equation} \label{eq:lower_bound_of_exceeding_1/2}
\Pb\left(g_{W_{s}}(\tilde{D}_{s-1})>\frac12 \Big| \tilde{D}_0>0\right) > a\,e^{-2(\log \log n)^2} > e^{-3(\log \log n)^2}.
\end{equation}
Comparing (\ref{eq:upper_bound_on_dying_out}) and (\ref{eq:lower_bound_of_exceeding_1/2}) we see that the conditional probability that $g_{W_s}(\tilde{D}_{s-1})$ reaches $1/2$ before absorbing at zero given $\tilde{D}_0 >0$ is at least $1-e^{-3(\log \log n)^2}$, finishing part (b) of the sketch.
Finally the time to either hitting 1/2 or absorbing at 0 is dominated by a geometric random variable with parameter $e^{-3(\log \log n)^2}$, so with probability going to $1$ this will happen before time $T_n$ whenever $T_n > e^{4(\log \log n)^2}$. Adding the fact that since $\epsilon_n > 1/n$ for $n$ large, $\Pb(\tilde{D}_0=0) < e^{-\sqrt{n}}$, this concludes the proof of (i). Now (ii) immediately follows from part (i) of Lemma \ref{lem:correlated_sens_stab}.

For (iii), observe that by (i) and since $\tilde{D}$ absorbs when it hits $0$ or $1$, we get by symmetry of $\tilde{D}$ that $\lim_{n \rightarrow \infty}\Pb(\tilde{D}_{T_n} = 0) = \lim_{n \rightarrow \infty}\Pb(\tilde{D}_{T_n} = 1) = 1/2$ for $\{T_n\}$ growing sufficiently large. Since $h_n$ is odd (iii) follows for such $\{T_n\}$.

\end{proof}

\begin{proof}[Proof of Theorem \ref{thm:stronger_correlated_noise_sensitivity}]
    The proof is very similar to the proof of Theorem \ref{th:stronger_noise_sensitivity_deep_networks}. A quick glance back at the proof of (i) in Theorem \ref{thm:correlated_noise_sensitivity} shows that $\tilde{D}$ with high probability hits $1/2$ before $0$ also if the input $(\omega,\omega^{\epsilon_n})$ is replaced with $(\omega,\eta)$. Now (i) follows from exactly the same proof as that of part (i) of Lemma \ref{lem:correlated_sens_stab}.
    Finally, both (ii) and (iii) follows from (i) in the same way as (ii) and (iii) follow from (i) in Theorem \ref{th:stronger_noise_sensitivity_deep_networks}.
\end{proof}
Next we move to proving that for sufficiently large $\rho$, $\{f_{n,T_n}\}$ is 
noise stable.
First we need an upper bound on $g_w(v)$:

\begin{lemma} \label{lem:upper_bound_g_w}
Let $w \in \R^2$ and let $(r,\theta)$ be its polar coordinates and let
$\varphi=\arctan\sqrt{v/(1-v)}$. Assume that
$|\theta| \geq \varphi$ and $|\theta - \pi| \geq \varphi$.
Then

\[g_w(v) \leq e^{-\frac12 r^2\sin^2(|\theta|-\varphi)}.\]
\end{lemma}

\begin{proof}

As in the proof of Lemma \ref{lem:lower_bound_2_g_w}, let
\[A_v = \left\{(x,y):|y| \leq \sqrt{\frac{v}{1-v}}|x|\right\},\]
so that
\[g_w(v)=\Pb((w^C+\xi,w^{C^c}+\eta) \in A_v),\]
where $\xi$ and $\eta$ are independent standard normal. Without loss of generality, we assume that $\theta \in [0,\pi/2]$. By assumption, $w \not \in A_v$. and the Euclidean distance from $w$ to $A_v$ is $r\,\sin(\theta-\varphi)$, so for $w+(\xi,\eta)$ to be in $A_v$, it is necessary that $\xi^2+\eta^2 \geq r^2\sin^2(\theta-\varphi)$. The left hand side is exponential with mean $2$, so
\[g_w(v) \leq e^{-\frac{1}{2}r^2\sin^2(\theta-\varphi)}.\]

\end{proof}

 \begin{theorem} \label{them:stay_low_and_die}
    
    Assume that $\rho > 1 - \log\log n/18\log n$. Then the following holds.
    
    \begin{itemize}
    \item[(i)] For $T=T_n \geq (\log n)^{1/4}$ and for all $\delta>0$, it holds for all sufficiently small $\epsilon>0$ that if $\tilde{D}_0=v<\epsilon^2/2$, then
    \[\limsup_n \Pb\left(\{\tilde{D}_T>0\} \cup \left\{\exists t \in \{1,\hdots,T\}:\tilde{D}_t>\frac{\epsilon^2}{2}\right\}\right) < \delta.\]
    In particular,
    \[\limsup_n \Pb\left(\exists t \geq 1:\tilde{D}_t>\frac{\epsilon^2}{2}\right) < \delta.\]
    \item[(ii)] $f_{n,T_n}$ is annealed and quenched noise stable if either $h_n$ has a sharp threshold at 1/2 or $T_n \geq (\log n)^{1/4}$.
    \end{itemize}
\end{theorem} 

    \begin{proof}

    This proof will start with showing that with high probability, $\Tilde{D}$ immediately drops to less than $1/(\log n)^{1/2}$ and then stays there for a time of order at least $(\log n)^{1/16}$. Having done that, it will be shown that during that time, $\Tilde{D}$ will in fact have hit $0$. 
    As in previous proofs, many inequalities used are obviously far from optimal, but simply good enough for their purpose.

    \smallskip
    
    Consider now the first step of $\tilde{D}$. Write $(r,\theta)$ for the polar coordinates of $W=W_1$. We have by Lemma \ref{lem:upper_bound_g_w}, since $\arctan\sqrt{(\epsilon^2/2)/(1-(\epsilon^2/2)^2)} < \epsilon$
    \begin{align*} 
    \Pb\left(g_W(v) < e^{-\frac{18\epsilon^3 \log n}{\log\log n}}\right) &> \Pb\left(g_W\left(\frac{\epsilon^2}{2}\right) < e^{-\frac{18\epsilon^3 \log n}{\log\log n}}\right) \\
    &> \Pb\left(e^{-\frac12r^2\sin^2(|\theta|-\epsilon)} < e^{-\frac{18\epsilon^3\log n}{\log \log n}}\right) \\
    &> \Pb\left(|\theta|>3\epsilon,|\theta-\pi|>3\epsilon,\frac12 r^2>\frac{18\epsilon\log n}{\log \log n}\right) \\
    &> \left(1-\frac{18\epsilon}{\pi}\right)e^{-\epsilon} \\
    &> 1-4\epsilon
    \end{align*}
    and given that $g_W(v) < e^{-\frac{18\epsilon^3 \log n}{\log\log n}}$, we get for large $n$,
    \[\tilde{D}_1<2e^{-\frac{18\epsilon^3\log n}{\log \log n}}<\frac{1}{(\log n)^{1/2}}\]
    with conditional probability at least $1-\epsilon$. Taken together, the last two inequalities give,
    \begin{equation} \label{ei}
    \Pb\left(\tilde{D}_1 \geq \frac{1}{(\log n)^{1/2}}\right) < 5\epsilon.
    \end{equation}
    Let $E_t = \{\tilde{D}_t < 1/(\log n)^{1/2}\}$; this notation will be convenient as we are going to do much conditioning on this event from here on.

    \smallskip 
    
    Next consider the distribution of $\tilde{D}_{t+1}$ given $\tilde{D}_t = \alpha^2/2 < 1/(\log n)^{1/2}$. Let as before $(r,\theta)=(r_t,\theta_t)$ be the polar coordinates of $W_t$, $t=1,2,\ldots$.
    
    Let 
    \[B=\left\{e^{-\frac12 r^2\sin^2(|\theta|-\alpha)}<\frac{1}{2(\log n)^{1/2}}\right\}.\]
    It is sufficient for $B$ to occur that $|\theta| \geq 3/(\log n)^{1/4}$ and $r^2 \geq (\log n)^{1/2}\log\log n$. Since $1-\rho < \log\log n/300\log n$, this gives for large $n$,
    
    \begin{equation} \label{en} 
    \Pb\left(B \, \Big| \, \tilde{D}_t = \frac{\alpha^2}{2}\right) = \Pb(B) > \left(1-\frac{6}{\pi(\log n)^{1/4}}\right)e^{-\frac{\log \log n}{36\log n}(\log n)^{1/2}\log\log n} > 1-\frac{3}{(\log n)^{1/4}}.
    \end{equation}
    
    When $B$ occurs, 
    \[g_{W}\left(\frac{\alpha^2}{2}\right) < \frac{1}{2(\log n)^{1/2}}\]
    according to Lemma \ref{lem:upper_bound_g_w}. Hence for $n$ large by Chernoff bounds,
    \[\Pb\left(\tilde{D}_{t+1}<\frac{1}{(\log n)^{1/2}} \Big|\tilde{D}_t = \frac{\alpha^2}{2},B\right) > 1-e^{-\frac{n}{\log n}} > 1-\frac{1}{(\log n)^{1/4}}.\]
    Since $\alpha^2/2$ is an arbitrary number in $[0,1/(\log n)^{1/2})$, we get
    \[\Pb\left(\tilde{D}_{t+1}<\frac{1}{(\log n)^{1/2}} \Big|\tilde{D}_t < \frac{1}{(\log n)^{1/2}},B\right) > 1-\frac{1}{(\log n)^{1/4}}\]
    and hence on combining with (\ref{en}),
    \begin{equation} \label{ej}
    \Pb \left(\tilde{D}_{t+1}<\frac{1}{(\log n)^{1/2}} \Big| E_t\right) > 1-\frac{4}{(\log n)^{1/4}}.
    \end{equation}
    
    Next let $A=A_{t+1}=\{|\theta_{t+1}|-\epsilon>\pi/4,r_{t+1}^2 > 6\log n\}$. Then since $1-\rho<\log \log n/18\log n$, provided that $\epsilon<\pi/12$,
    \[\Pb(A) > \frac13 e^{-\frac{\log\log n}{6}} > \frac{2}{(\log n)^{1/5}}.\]
    If $A$ occurs, Lemma \ref{lem:upper_bound_g_w} implies for $v<\epsilon^2/2$ 
    \[g_W(v) < e^{-\frac14 r^2} < e^{-\frac32 \log n} = \frac{1}{n^{3/2}}.\]
    That entails in turn that
    \[\Pb(\tilde{D}_{t+1}=0|A) > \left(1-\frac{1}{n^{3/2}}\right)^n > 1 - \frac{1}{\sqrt{n}}.\]
    Hence
    \begin{align}
    \Pb\left(\tilde{D}_{t+1}=0|E_{t}\right) &> \Pb\left(\tilde{D}_{t+1}=0 | A \cap E_t\right) \Pb\left(A|E_{t-1}\right) 
    = \Pb\left(\tilde{D}_{t+1}=0 | A\right) \Pb\left(A\right) \nonumber \\
    &> \frac{2}{(\log n)^{1/5}}\left(1-\frac{1}{\sqrt{n}}\right)
    > \frac{1}{(\log n)^{1/5}}. \label{el}
    \end{align}

    \smallskip
    
    By taking (\ref{ej}) and (\ref{el}) together, it now easily follows that the conditional probability given $E_1$ that $\tilde{D}$ for $t \geq 1$ hits $0$ before $(\log n)^{1/2}$ exceeds $1-1/(\log n)^{1/4-1/5} = 1-1/(\log n)^{1/20}$ and with probability exceeding $1-e^{-(\log n)^{1/20}}$ it happens before time $(\log n)^{1/4}$.
    
    Combining this with (\ref{ei}) and taking $\epsilon < \delta/6$ we get
    \[\limsup_n \Pb\left(\{\tilde{D}_{T_n}>0\} \cup \{\exists t \in \{1,\hdots,T_n\}:\tilde{D}_t>\frac{1}{(\log n)^{1/2}}\}\right) < \delta.\]

    This obviously implies
    \[\limsup_n \Pb\left(\{\tilde{D}_{T_n}>0\} \cup \{\exists t \in \{1,\hdots,T_n\}:\tilde{D}_t>\epsilon\}\right) < \delta.\]
    
    For (ii) assume first that $h_n$ has a sharp threshold at 1/2 and $T_n$ is arbitrary.  Fix $\kappa>0$. Then (i) tells us that for sufficiently small $\epsilon>0$ and sufficiently large $n$, 
    $P(\forall t:\tilde{D}_t < 2\epsilon) > 1-\kappa/2$. In particular 
    \[\limsup_n\Pb(\tau_{2\epsilon}^\epsilon \leq T_n) < \frac{\kappa}{2}.\] 
    Since the final output is $f_{n,T_n}(\omega)= h_n(\omega_T)$ and $h_n$ has a sharp threshold at 1/2, the result follows from Lemma \ref{lem:correlated_sens_stab}.
    
    Finally if $T_n \geq (\log n)^{1/4}$, then by (i), $\lim_{\epsilon \rightarrow 0}\Pb(\tilde{D}_{T_n}=0) = 1$, i.e.\ $\lim_{\epsilon \rightarrow 0}\Pb(\omega_{T_n} = (\omega^\epsilon)_{T_n}) = 1$, which implies noise stability.

    \end{proof}

\section{Convolutional treelike networks} \label{sec:conv}

    Above we have studied a Boolean representation of randomised feed forward neural networks. Another common architecture of neural networks is convolutional neural networks, ConvNet, where the output of each layer is the convolution between the input and some filter of size $d$ through some activation function. The learnable parameters in the model will be those in the filter. If we study a one dimensional input and if no padding is used, the output of each convolution is of length $(n - d)/s + 1$ where $s$ is the stride and $n$ the number of inputs. As previously, we consider the activation function to be the sign function. We also limit ourselves to a filter size of length $d=3$ (but the technique used will with a few observations easily generalise to some other settings, which we point out and state at the end of this section). Stacking many of these convolutional layers our networks can be seen as a graph $G_n = (V_n,E_n)$. 
    From this graph we induce a Boolean function $f_n$ that goes from $\{-1,1\}^n$ to \{-1,1\}. 
    The graphs that we consider here, and soon give proper definitions of, can be seen in Figure \ref{fig:3-maj}.
    
    With filter size 3, the value of each node at layer $t$ in $G_n$ is the weighted $3$-majority function, where the weights are the parameters in the filter $\theta_t$. These weights are the same for all triples of nodes next to each other on the same layer on which the filter is applied. A weighted $3$-majority function can only express either a regular $3$-majority, negative $3$-majority, a dictator function or a negative dictator function. Clearly, reversing the signs of all values in a layer does not have anything to do with stability questions, so we may assume that the filter at a given layer either expresses a dictator or a regular majority. In each setting analysed here, we will at first assume that each layer expresses a regular majority, i.e.\ $\theta_t = (1,1,1)$ $\forall t$. 
    It will then be easy to see that the arguments used can be easily extended to the setting where some layers are dictator layers under the very mild assumption that the distribution of the $\theta_t$'s is such that there is a probability bounded away from 0 that $\theta_t$ expresses a regular majority. Each $\theta_t$ is here assumed to be independent across $t$. 
    First, notice that if the stride $s > 2$, $G_n$ would be the iterated $3$-majority function with no overlap, which is known to be sensitive, so we will only consider $s=1$ and $s=2$.

    This leaves us with four different structures of interest. There corresponding graph $G_n$ are illustrated in Figure \ref{fig:3-maj} and have the following formal definitions of them: 
    
    \begin{figure}[h]
        \centering
        \begin{subfigure}[b]{0.4\linewidth}
        \includegraphics[width = 1\textwidth]{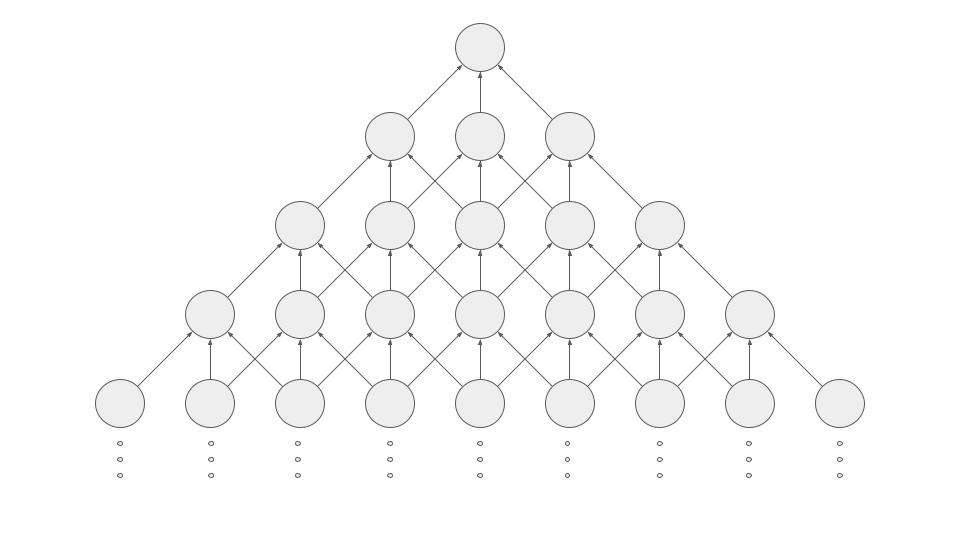}
        \caption{\label{fig:stride1} Convolutional iterated $3$-majority function with stride $1$.}
        \end{subfigure}
        \begin{subfigure}[b]{0.4\linewidth}
        \includegraphics[width = 1\textwidth]{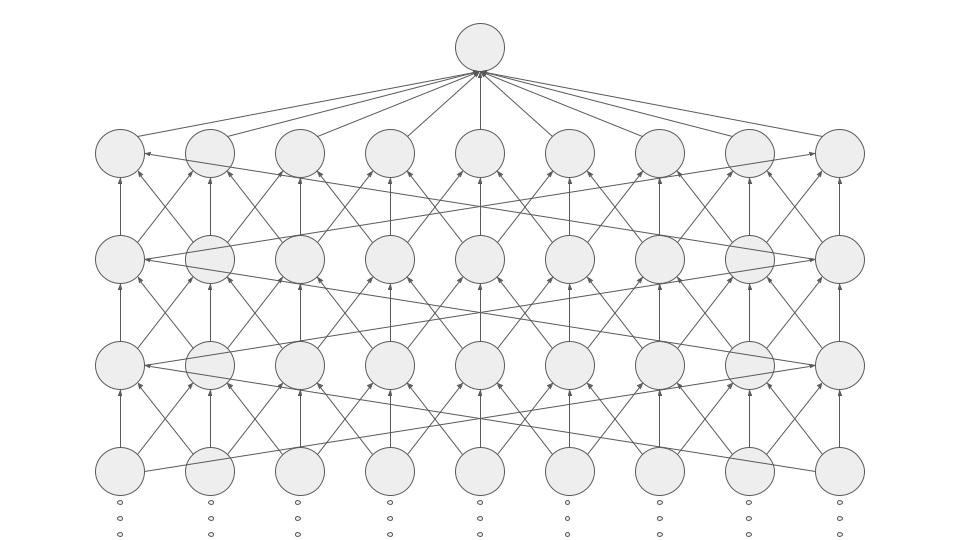}
        \caption{\label{fig:stride1c} Convolutional iterated $3$-majority function with stride $1$ on an n-cycle.}
        \end{subfigure}
        \begin{subfigure}[b]{0.4\linewidth}
        \includegraphics[width = 1\textwidth]{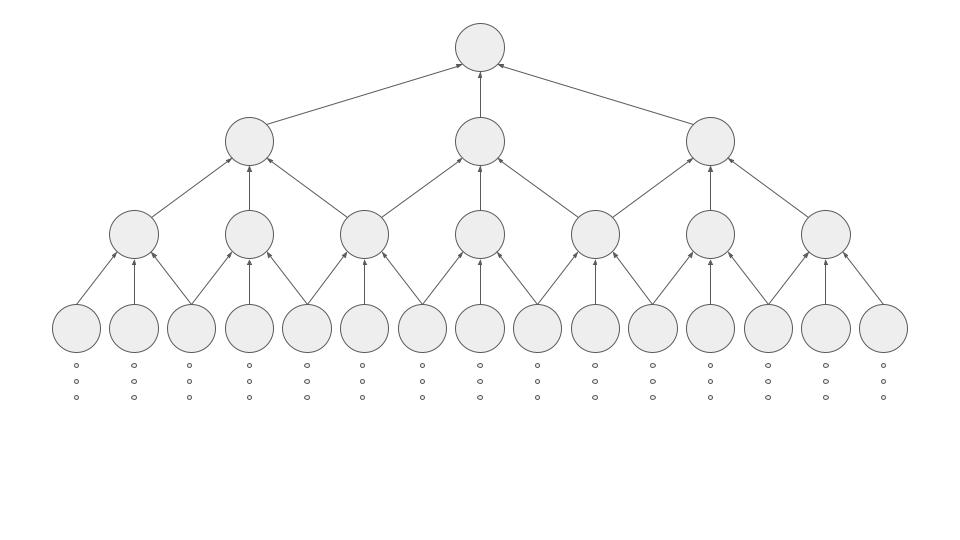}
        \caption{\label{fig:stride2} Convolutional iterated $3$-majority function with stride $2$.}
        \end{subfigure}
        \begin{subfigure}[b]{0.4\linewidth}
        \includegraphics[width = 1\textwidth]{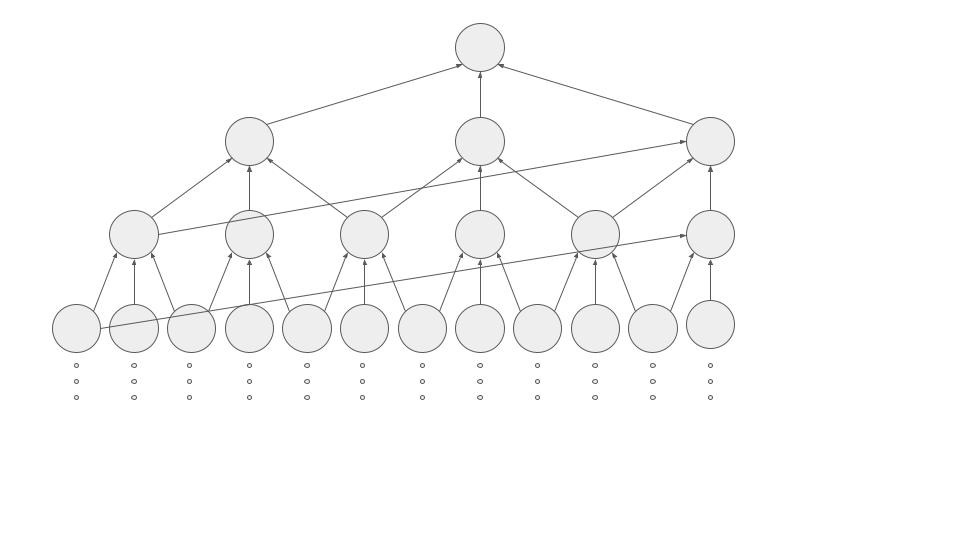}
        \caption{\label{fig:stride2c} Convolutional iterated $3$-majority function with stride $2$ on an n-cycle.}
        \end{subfigure}
        \caption{The graphs $G_n^{(1)}$,$G_n^{(1')}$,$G_n^{(2)}$ and $G_n^{(2')}$.}
        \label{fig:3-maj}
    \end{figure}

    \begin{itemize}
    \item[(a)] Convolutional iterated $3$-majority with stride $1$, $G^{(1)}_n=(V,E)$, where \newline $V=\{v_{n,0},v_{n-1,-1},v_{n-1,0},v_{n-1,1},v_{n-2,-2},v_{n-2,-1},v_{n-2,0},v_{n-2,1},v_{n-2,2}, \hdots,v_{0,-n},\hdots,v_{0,n}\}$ and $E=\{(v_{k,i},v_{k-1,j}): k=n,\hdots,1,\, j=i-1,i,i+1\}$.
    \item[(b)] Convolutional iterated $3$-majority with stride $1$ on an n-cycle, $G^{(1')}_n=(V,E)$, where \newline $V=\{v_{K,-n},v_{K-1,-(n-1)},\hdots,v_{K-1,n},\hdots,v_{0,-n},\hdots,v_{0,n}\}$, $E=\{(v_{K,1},v_{K-1,i}):i=1,\hdots,n\} \cup \{v_{k,i},v_{k-1,j}: k=K-1,\hdots,1,\,|i-j| \leq 1 \mod n\}$.
    \item[(c)] Convolutional iterated $3$-majority with stride $2$, $G^{(2)}_n=(V,E)$, where \newline $V=\{v_{n,1},v_{n-1,1},v_{n-1,2},v_{n-1,3},\hdots,v_{0,1},\hdots,v_{0,2^{n+1}-1}\}$, $E=\{(v_{n-k,i},v_{n-k-1,j}):k=0,\hdots,n-1,\, |j-2i| \leq 1\}$.
    \end{itemize}
    
    The graph in (d) will not be treated any more than that in the end it will be obvious that the results for (c) are valid there too, so the formal definition is skipped.

    Given one of these $G_n$, let $f_n$ be the Boolean function induced by $G_n$ where the value of each node $v$, $f(v)$, is evaluated as the majority of the evaluation of the three connected nodes below it (where ``below'' refers to the standard tree convention of thinking of nodes $v_{k+1,i}$ as sitting immediately below the nodes $v_{k,i}$). The nodes of layer $0$ are considered to output the input $\omega$ into the ConvNet, where $\omega(j)$, $j \in L_0$, corresponds to bit $j$ in the Boolean input vector. For example, for a node $v_{k,i} \in G_n^1$, $f(v_{k,i})$ can be evaluated recursively using $f(v_{k,i}) = \sign \left(f(v_{k-1,i}) + f(v_{k-1,i+1}) + f(v_{k-1,i+2}) \right)$ and $f(v_{0,j}) = \omega(j)$, $\forall k,j: \:  1 \leq k \leq n$, $1 \leq j \leq 2n + 1$. We will denote the set of nodes at layer $k$ as $L_{k}$, $k = 0,\hdots,n$. In terms of the formal definition, $L_k=\{v_{k,i}:v_{k,i} \in V\}$. Note that the layers/generations are numbered bottom up, which is unconventional in the graph sense, but is the ``right thing'' to do in the neural net sense.
    For any $i$ such that $v_{k,i}$ and $v_{k,i+1}$ are both in $V$, we say that those two vertices are next to each other or a closest pair or, sometimes, neighbours in $L_k$ even though they are strictly speaking not neighbours in $G_n$. For $G^{(1')}_n$ and $G^{(2')}_n$, the nodes $v_{k,1}$ and $v_{k,n}$ are also said to be next to each other.
    
    For each node $u \in G_n$, in layer $k$, let $D_u$ denote its set of descendants, i.e.\ the set of nodes in layers $k-1,\hdots,0$ defined recursively that the descendants in $L_{k-1}$ are $u$'s three neighbours (in the graph sense) in $L_{k-1}$ and then the descendants in $L_{k-j}$ are the set nodes there that have an edge to at least one descendant in $L_{k-j-1}$. Observe that to determine $f(u)$ it is sufficient to study the subgraph $D_u$. When $v$ is a descendant of $u$, we equivalently say that $u$ is an ancestor of $v$. Let $A_u$ be the set of ancestors of $u$, i.e.\ $A_u$ is the set of nodes $v$ such that $u \in D_v$. Write $D_{u,k} = D_u \cap L_k$, 
    and $A_{u,k}=A_u \cap L_k$,  and write $A^+_{u,k}$ for the union of $A_{u,k}$ and the set of nodes in $L_k$ that are next to a node in $A_{u,k}$.
    If $v \in A_u$, we say that $v$ is a parent of $u$ or that $u$ is a child of $v$ if the graphical distance between $u$ and $v$ is 1.
    Also define $D_\Lambda = \cup_{u \in \Lambda} D_u$ for a set of nodes $\Lambda$. 
    
    The following subsections prove the different noise properties of $f_n$ induced by the different convolutional graphs. 
    
\subsection{Convolutional iterated 3-majority with stride $1$} \label{sec:stride1}

    Let $G_n = G^{(1)}_n = (V_n,E_n)$ i.e.\ the $3$-iterated majority network with stride $1$, Figure \ref{fig:stride1}. Then the induced Boolean functions $\{f_n\}$ requires $N = 2n + 1$ input bits which we label in accordance with how the network is defined: $\omega(-n),\hdots,\omega(n)$. 
    
    \smallskip 
    
    Write $f=f_n$. The key observation to make is that if $f(v_{0,i})=\omega(i)$ and $f(v_{0,i+1})=\omega(i+1)$ are equal, then $f(v_{t,i})=f(v_{t,i+1})$ for all $t$ such that $v_{t,i}$ and $v_{t,i+1}$ exist. If only one of them, say $v_{t,i}$ exists, then $f(v_{t,i})=f(v_{t-1,i})=f(v_{t-1,i+1})$. In particular if $\omega(-1)=\omega(0)$ or $\omega(i+1)$ or $\omega(0)=\omega(1)$, then $f(\omega)=f(v_{n,0})=\omega(0)$. Obviously this generalises the statement that if $f(v_{s,0})$ and $f(v_{s,1})$ are equal, then $f(v_{t,0})=f(v_{t,1})$ for all $t \geq s$.
    
    If we instead have $\omega(-1) \neq \omega(0) \neq \omega(1)$ and $\omega(2)=\omega(1)$, then $f(v_{1,0})=f(v_{1,1})=\omega(1)=\omega(2)$ and it then follows in the same way that $f(\omega)=\omega(1)$.
    
    An inductive structure suggests itself. Let $K$ be the smallest positive integer such that either $\omega(-K)=\omega(-(K-1)) \neq \omega(-(K-2) \neq \ldots \neq \omega(K-1)$ or $\omega(-(K-1)) \neq \omega(-(K-2) \neq \ldots \neq \omega(K-1) = \omega(K)$, if such an $K$ exists. In other words $K$ is the distance from $0$ to a closest pair of input bits with the same value.
    (There may be two such pairs. If so, by parity $\omega(-K)=\omega(-(K-1))=\omega(K-1)=\omega(K)$ with the input alternating between $-(K-1)$ and $K-1$.)
    If no such $i$ exists, set $K=n+1$.
    
    Assume without loss of generality that, unless $K=n+1$, $\omega(K-1)=\omega(K)$. It follows from the definition of $K$, $\omega(-(K-1)) \neq \omega(-(K-2)) \neq \ldots \neq \omega(K-2) \neq \omega(K-1)$, This clearly holds also for $K=n+1$.
    If $K \leq n$, we then get $f(v_{1,-(K-2)}) \neq f(v_{1,-(K-3)}) \neq \ldots \neq f(v_{1,K-3}) \neq f(v_{1,K-2}) = f(v_{1,K-1})$. This means that in layer $1$, the closest pair of input bits with the same value is one step closer to $0$ than in layer $0$. It now follows from induction that $f_n(\omega)=f(v_{n,0}) = \omega(K)$. Of particular importance here is to observe that in particular $f_n(\omega)$ does not depend on $\omega(j)$, $j \not \in [-K,K]$.
    
    It remains to understand what happens if $K=n+1$, i.e\ when the whole input is alternating: $\omega(-n)=\omega(-(n-1))= \ldots \neq \omega(n-1) \neq \omega(n)$. However then $f$ will be alternating at all layers and $f(\omega)=\omega(-n)=\omega(n)$.
    
    We have established the following lemma.

    \begin{lemma}{(Closest pair lemma)}
        \label{lemma:barrier_function}
        
        The induced Boolean function $f = \mbox{"iterated 3-majority with stride 1"}$ on $G_n^{(1)}$ is the same as the function $g$ with the following description:

        Given the input bit vector $\omega$ of length $2n + 1$, let $K$ be the
        smallest positive integer $i$ such that $\omega(-i) = \omega(-(i-1))$ or $\omega(i-1)=\omega(i)$ and set $g(\omega)=\omega(-K)$ or $g(\omega)=\omega(K)$ in the respective cases. If no such pair exists, set $g(\omega)= \omega(n)$. 
    \end{lemma}   

With this lemma we can state the following theorem.

\begin{theorem} \label{th:stride1_non_c}
    The iterated 3-majority function with stride $1$ is noise stable.
\end{theorem}
\begin{proof}
    According to Lemma \ref{lemma:barrier_function} we can translate the iterated 3-majority with stride $1$ function to the closest pair function. 
    
    Let $K$ be the distance between the closest pair and the mid input vertex $0$. Since each bit is i.i.d., $K$ is geometric with parameter $3/4$ truncated at $n+1$: 
    \[\Pb(K = k) = \frac{3}{4} \left(\frac{1}{4} \right)^{k-1}  = 3 \left(\frac{1}{4}\right)^k,\]
    $k=1,\ldots,n$ and $\Pb(K=n+1) = (1/4)^n$.
    
    Since $f_n$ does not depend on $\omega$ outside $[-K,K]$, for $f_n( \omega^\epsilon) \neq f_n(\omega)$ to hold there must at least be a disagreement in the interval $[-K,K]$, i.e.\ a $j \in [-K,K]$ with $\omega(j) \neq \omega^\epsilon(j)$. This gives  
    
    \[\Pb (f_n( \omega) \neq f_n(\omega^\epsilon)|K = k) \leq 1 - (1-\epsilon)^{2k + 1},\]
    
    $k=1,\ldots,n+1$ and $\Pb((f_n( \omega) \neq f_n(\omega^\epsilon)|K = n+1) \leq 1-(1-\epsilon)^{2n+1}$.
    
    Putting this into Definition (\ref{noise_stable}) we see that 
    
    \begin{align*}
        \lim_{\epsilon \rightarrow 0} \limsup_n \Pb(f_n( \omega^\epsilon) \neq f_n(\epsilon)) &\leq \lim_{\epsilon \rightarrow 0} \limsup_n \sum_{k = 1}^n 3 \left( \frac{1}{4} \right)^k \left( 1 - (1-\epsilon)^{2k+1} \right)\\
        &\leq \lim_{\epsilon \rightarrow 0} \sum_{k = 1}^\infty 3 \left( \frac{1}{4} \right)^k \left( 1 - (1-\epsilon)^{2k+1}\right) \\
        &= \lim_{\epsilon \rightarrow 0} 3 \left[ \frac{4}{3} - \frac{4 (1-\epsilon)^4 }{4 - (1-\epsilon)^2}\right] = 0
    \end{align*}
    
    This concludes the proof.
    
\end{proof}

Let us now consider when the filter weights $\theta_t$ are random. Then $\theta_t$ represents either a regular majority or a dictator function. If $\theta_t$ represents a dictator function each node $u$ at layer $t$ only depends on one node $v$ at layer $t-1$ and the values $f_n(u)$ and $f_n(v)$ are the same. This means we can effectively skip that layer on noticing that either the leftmost or the rightmost node of layers $t-1,t-2,\ldots,0$ can no longer affect anything and can be removed.
Removing all dictator layers results in an ordinary stride 1 iterated 3-majority with fewer layers. Hence Theorem \ref{th:stride1_non_c} applies and we can conclude the following. 

\begin{corollary}
    The iterated 3-majority function with stride 1 is annealed and quenched noise stable under any probability distribution on $\theta_t$. 
\end{corollary}

\subsection{Convolutional iterated 3-majority with stride $1$ on the $n$-cycle} \label{sec:strid1n}

    Consider now $G^{(1')}_n$, Figure \ref{fig:stride1c}, the network where each layer has $n$ nodes with circular convention so that the first node at each layer is considered next to the last node in that layer. We assume that $n$ is odd. Each output from node $i$ in layer $t$ is the majority of nodes $i-1$, $i$, $i+1$ in layer $t-1$. The final output of the network is the majority of $\omega_T$ for some $T = T_n$. The starting configuration is $\omega_0 \in \{-1,1\}^n$. 
    \begin{theorem}
    For any $T=T_n$, the convolutional iterated 3-majority function with stride 1 on an n-cycle is noise stable.
    \end{theorem}
    \begin{proof}
    
    Let $\ldots,X(-1),X(0),X(1),X(2),\ldots$ be independent Bernoulli(1/2) random variable and model $(\omega,\omega^\epsilon)$ in terms of these by taking $\omega=\omega_0=(X(1),\ldots,X(n))$ and $\omega^\epsilon=(X^\epsilon(1),\ldots,X^\epsilon(n))$. Let $\tau_0 = \min\{i>1:X(i-1)=X(i)=X^\epsilon(i-1)=X^\epsilon(i)\}$. Then recursively, let
    \[\tau_{j} = \min\{i>\tau_{j-1}+1: X(i)=X(i-1)=X^\epsilon(i-1)=X^\epsilon(i)=-X(\tau_{j-1})\},\, j=1,2,\ldots.\]
    Observe that if one writes $X_0=X$, one can then define $X_1,X_2,\ldots$ in complete analogy with how $\omega_1,\omega_2,\ldots$ are defined, i.e.\ $X_{t+1}(i)=\sign(X_t(i-1)+X_t(i)+X_t(i+1))$, $t=0,1,2,\ldots$. 
    
    Let $V_0=[1,\tau_0]$ and then $V_j=[\tau_{j-1} +1,\tau_j]$. Let $S=\max\{j \geq 0:\tau_j \leq n\}$ (with $S=-1$ if $\tau_0>n$). Let also $\tilde{V}=[\tau_S+1,n]$ (with $\tau_{-1}$ taken to be $0$). Note that $\tilde{V} \subseteq V_{S+1}$ (and $\tilde{V}$ may also be empty). In short, this divides $[1,n]$ into chunks, where each chunk ends with two consecutive indexes $i-1$ and $i$ where $X_t(i-1)=X_t(i)=X^\epsilon_t(i-1)=X^\epsilon_t(i)$ for all $t$ plus a chunk $\tilde{V}$ at the end, which is empty precisely if $\tau_S=n$. On the circle, i.e.\ where $\omega$ and $\omega^\epsilon$ are defined, $V_0$ and $\tilde{V}$ are next to each other in a natural way.
    
    Note that $X_T$ and $\omega_T$ are equal on each $V_j$, $j=1,\ldots,S$. They may differ on $V_0$ and $\tilde{V}$, but this difference will not need to be controlled in any other way than the simple observation that their cardinalities are bounded.  
    
    Stability will be proven by proving that, for sufficiently small $\epsilon>0$, the number of bits where $\omega_T$ and $(\omega^\epsilon)_T$ differ is with high probability smaller than the difference between the number of $1$'s in $\omega_T$ and $n/2$.
    
    For subsets $I$ of $[n]$ and $y \in \{-1,1\}^n$, let $y(I) = (y(i))_{i \in I}$. Let $s(y(I))=\sum_{i \in I}y(i)$ and $s(y)=s(y([n]))$. Let $C_j=s((\omega^\epsilon)_T(V_j))-s(\omega_T(V_j))$, $j=0,1,2,\ldots,S$, $C_{\tilde{V}}=s((\omega^\epsilon)_T(V))-s(\omega_T(\tilde{V}))$ so that $C_{\tilde{V}}+\sum_{j=0}^S C_j = s((\omega^\epsilon)_T)-s(\omega_T)$. Let also $C'_j=s((X^\epsilon)_T(V_j))-s(X_T(V_j))$, $j=1,2,\ldots$ and note that the $C'_j$'s are i.i.d.\ and that $C_j=C'_j$ for $j=1,\ldots,S$, so that $s((\omega^\epsilon)_T(V_j))-s(\omega_T(V_j))$ also equals $C_0+C_{\tilde{V}}+\sum_{j=1}^S C'_j$.
    
    Clearly all moments of $|V_j|$ are uniformly bounded. Let $\nu=\E[|V_1|]$ and let $m=\lfloor n/\nu \rfloor$. 
    Fix arbitrarily small $\rho>0$ and $\delta>0$. By the Central Limit Theorem, there is a constant $K_1 < \infty$ independent of $n$ such that with $K^- = m-K_1\sqrt{n}$ and $K^+ = m+K_1\sqrt{n}$, for large $n$
    \[\Pb\left(\sum_{j=0}^{K^-}|V_j|  \geq n \right) < \frac18 \delta,\,\,\Pb\left(\sum_{j=0}^{K^+}|V_j| \leq n \right) < \frac18 \delta\]
    so that
    \begin{equation} \label{ea}
    \Pb(K^- \leq S \leq K^+) > 1-\frac14 \delta.
    \end{equation}
    The $C'_j$'s by symmetry have mean 0. 
    Take some $j \geq 1$ and fix it until (\ref{eq}). Let $F=F_j$ be the event that there exists an index $i \in V_j$ with $X(i) \neq X^\epsilon(i)$. On $F_j^c$, $X_T(i) = (X^\epsilon)_T(i)$ for all $T$ and $i \in V_j$. Thus $C'_j=0$ on $F^c$.
    Let 
    \[\chi_\epsilon = \chi^\epsilon_j =\min\{r>0: X(\tau_{j-1}+2r) \neq X^\epsilon(\tau_{j-1}+2r) \mbox{ or } X(\tau_{j-1}+2r+1) \neq X^\epsilon(\tau_{j-1}+2r+1)\}\]
    \[\chi=\chi_j=\min\{r>0: X(\tau_{j-1}+2r) = X^\epsilon(\tau_{j-1}+2r) = X(\tau_{j-1}+2r+1) = X^\epsilon(\tau_{j-1}+2r+1) \neq X(\tau_{j-1})\}.\]
    Then $\chi$ and $\chi_\epsilon$ are geometric random variables that cannot take on the same value and $\chi \geq |V_j|/2$. 
    Also, $F \subset G := \{\chi_\epsilon < \chi\}$ and it is standard that $G$ is independent of $\min(\chi,\chi_\epsilon)$ and
    \[\Pb(G) = \frac{2\epsilon-\epsilon^2}{2\epsilon-\epsilon^2+\frac14(1-\epsilon)^2} < 10\epsilon.\]
    Hence
    \begin{align*}
        \Var(C'_j) &\leq \E[{C'_j}^{2}] \leq
        4E[|V_j|^2 \mathbf{1}_{F}]
        \leq 4\E[\chi^2\mathbf{1}_{G}] \\
        &= 4\E[\E[\chi^2 \mathbf{1}_G|\min(\chi,\chi^\epsilon)]] \\
        &= 4\E[\Pb(G|\min(\chi,\chi^\epsilon))\E[\chi|\min(\chi,\chi^\epsilon),G]] \\
        &=4\Pb(G)\E[(\min(\chi,\chi^\epsilon)+Y)^2] \\
        &< 40\epsilon \E[(\min(\chi,\chi^\epsilon)+Y)^2],
    \end{align*}
    where $Y$ is a copy of $\chi$ that is independent of $(\chi,\chi_\epsilon)$. This gives
    \begin{equation} \label{eq} 
    \Var(C'_j) < 160\epsilon \E[\chi^2] < 16000\epsilon.
    \end{equation}

    It follows that 
    \[\Var\left[ \sum_{j=1}^{K^-} C'_j \right] < 16000 \epsilon m.\]
    By the Central Limit Theorem, for sufficiently large constant $M_5$
    \[\Pb\left(\left| \sum_{j=1}^{K^-} C'_j \right| > M_5\epsilon^{1/2}\sqrt{n}\right) < \frac{\delta}{4}.\]
    Taking $\epsilon$ sufficiently small, we get
    \begin{equation} \label{eb}
    \Pb\left(\left| \sum_{j=1}^{K^-} C'_j \right| < \frac12 \rho \sqrt{n} \right) > 1-\frac14\delta.
    \end{equation}
    
    Kolmogorov's inequality gives
    \begin{equation} \label{ec}
    \Pb\left(\max_K \Big| \sum_{j=K^-}^{K} C'_j \Big| > n^{1/3}\right) < \frac14 \rho
    \end{equation}
    for $n$ large.
    Adding that $\Pb(|C_0+C_{\tilde{V}}|>a_n) \leq \Pb(|V_0|+|\tilde{V}|>a_n) \rightarrow 0$ for any $a_n \rightarrow \infty$ to (\ref{ea}), (\ref{eb}) and (\ref{ec}) and summarising and recalling $C_j=C'_j$ for $1 \leq j \leq S$, we get
    \begin{equation} \label{eq:small_difference}
    \Pb\left( \, \left|C_0 + C_{\tilde{V}} + \sum_{j=1}^S C'_j\right| < \rho \sqrt{n}\right) = \Pb\left(|s(\omega^\epsilon)_T)-s(\omega_T) | <\rho \sqrt{n} \right)> 1-\delta
    \end{equation}
    for $\epsilon$ sufficiently small and $n$ sufficiently large.
    
    \medskip 
    The second part is very similar, but slightly easier.
    Let $\xi^+_{-1}=0$. Define for $j=0,1,2,\ldots$ recursively
  \[\xi_j^-=\min\{i>\xi_{j-1}^+ +1:X_0(i-1)=X_0(i)=-1\},\, \xi_j^+=\min\{i>\xi_j^- +1:X_0(i-1)=X_0(i)=1\}.\]
  Let $U_j=[\xi^+_{j-1}+1,\xi^+_j]$, $j=0,1,2,\ldots$, and let $\mu=\E[|U_j|]$. Let $R=\max\{j:\xi^+_j \leq n\}$. Let also $\tilde{U} = [\xi^+_R+1,n]$. Since $|U_j|$ is clearly bounded stochastically by two times a sum of two independent geometric(1/4) random variables, all moments of $|U_j|$ are finite. 
  Let $m=\lfloor n/\mu \rfloor$ and fix an arbitrarily small $\delta>0$. The Central Limit Theorem gives that for a sufficiently large constant $K_2$, 
  \begin{equation} \label{ed}
  \Pb(L^- < R < L^+)>1-\frac14 \delta,
  \end{equation}
  where $L^-=m-K_2\sqrt{n}$ and $L^+=m+K_2\sqrt{n}$.
  Let $D_j=s(\omega_T(U_j)) = \sum_{i=\xi_{j-1}+1}^{\xi_j} \omega_T(i)$, $j=0,1,2,\ldots,R$, $D_{\tilde{U}}=s(\omega_T(\tilde{U}))$ and $D'_j=s(X_T(U_j))$, $j=1,2,\ldots$. The $D_j$ are i.i.d.\ with mean 0 (by symmetry) and, since $|D_j| \leq |U_j|$, finite and clearly nonzero variance. Since $X_T=\omega_T$ on $U_1,\ldots,U_R$, it also holds that $D'_j=D_j$ for $j=1,\ldots,R$ and $s(\omega_T)=D_0+D_{\tilde{U}}+\sum_{j=1}^RD'_j$. 
  By the Central Limit Theorem and sufficiently small $\rho>0$,
  \begin{equation} \label{ee}
  \Pb\left( \left| \sum_{j=1}^{L^-} D'_j \right| > 3\rho\sqrt{n} \right) > 1-\frac14 \delta.
  \end{equation}
  Also, for large $n$ by Kolmogorov's inequality
  \begin{equation} \label{ef}
  \Pb\left(\max_L \Big|\sum_{j=L^-}^{L} D'_j \Big| < n^{1/3} \right) > 1-\frac14 \delta.
  \end{equation}
  Summing up (\ref{ed}), (\ref{ee}) and (\ref{ef}), we get for large $n$ and sufficiently small $\rho>0$,
  \begin{equation} \label{eg}
  \Pb\left(\Big| \sum_{j=1}^R D_j \Big| > 2\rho\sqrt{n}\right) > 1-\frac34\delta.
  \end{equation}
  Also, for any $a_n \rightarrow \infty$, $P(|U_1|+|U_{R+1}|<a_n) \rightarrow 1$. This gives in conjunction with (\ref{eg}) for sufficiently large $n$,
  \[\Pb\left(|s(\omega_T)| > 2\rho\sqrt{n}\right) > 1-\delta.\]
  Taking $\rho$ also small enough to satisfy (\ref{eq:small_difference}), we get for sufficiently large $n$
  \[\Pb(\maj(\omega_T) \neq \maj((\omega^\epsilon)_T)) < 2\delta.\]
  This proves noise stability.
  
   \end{proof}
   
   {\bf Remark.} With similar arguments as for the non-cycle case, it is easy to see that the outputs from each layer will soon freeze in a configuration that is easily determined by the following observation. Suppose $\omega_0(k-1) = \omega_0(k) = 1$ and let $l = \min\{j \geq k+2 \: : \: \omega_0(j-1) = \omega_0(j) = -1\}$. Let also $m = \max\{k+2 \leq j \leq l-2 \: : \: \omega_0(j-1) = \omega_0(j) = 1\}$. Then for $t \geq d$ (note that $l-m$ is even), we have $\omega_t(k+1) = \hdots = \omega_t((l+m-2)/2) = 1$, $\omega_t((l+m)/2) = \hdots = \omega_t(l) = -1$, where $d = (1/2)\max\{j-i:k+2 \leq i \leq j \leq \ell: \omega_0(i)=-1,\omega_0(i+1)=1,\hdots,\omega_0(j-1)=-1,\omega_0(j)=1\}$. Of course the analogous thing with signs reversed holds.  
    Taking $\bar{d}$ as the maximum of all such $d$'s over the whole input, $\omega_t = \omega_{\bar{d}}$ for all $t \geq \bar{d}$. It is easy to see that $\bar{d}$ is of order $\log n$.

    \smallskip

    As in Section \ref{sec:stride1}, the result can be generalised to random $\theta_t$. This is easily done by noticing that if $\theta_t$ represents a dictator function, each note at layer $t$ only has one input from layer $t-1$ with no overlap between nodes. This means that we can effectively skip that layer and indeed the whole graph can be collapsed into a less deep graph with the same width where each layer expresses regular 3-majority. The result is summarised in the following corollary.
    \begin{corollary}
        The iterated 3-majority with stride 1 and with random $\theta_t$ chosen according to any probability distribution on the $n$-cycle is quenched and annealed noise stable. 
    \end{corollary}
    
\subsection{Convolutional iterated 3-majority with stride $2$} \label{sec:stride2}

    Here $G_n=G^{(2)}_n$, i.e. the graph in Figure \ref{fig:stride2}, and $f_n$ induced by $G_n$ are considered, but as we will see in the end, the results are easily extended to $G_n^{(2')}$. 
    There are a few crucial observations to make about $G_n$. First that at each layer $k$, $|A_{j,k}| \leq 2$ for every input node $j \in L_0$. This follows, by the definition of $G_n$, from the easily checked fact that the union of the sets of parents of two neighbouring nodes are either two neighbouring nodes or a single node. Secondly, for two nodes $u,u' \in L_k$ $D_u \cap D_{u'} = \emptyset$ if $u$ and $u'$ are not next to each other since the set of children of $u$ and $u'$ are disjoint and no child of $u$ is next to any child of $u'$. This means $f(u)$ and $f(u')$ are measurable with respect to two disjoint subsets of $\omega$ and are hence independent. We can now formulate the following lemma.
    
    \begin{lemma}
    \label{lem:stride2}
    Fix an arbitrary layer $L_k$ and let $u_j=v_{k,j}$ be the nodes in $L_k$ enumerated from left to right. Also, let $S \subseteq \{1,\hdots,|L_k|\}$. Then
    \[\Pb(\forall j \in S: f(u_j)=1) \geq \frac{1}{2^{|S|}}.\]
    \end{lemma}
    
    \begin{proof}
    
    Since $f(u_j)$ is an increasing function of $\omega$, this follows from Harris inequality.
    
    \end{proof}
    
    Now, let $\CA$ be some randomised algorithm to determine the value of $f_n(\omega)$ by querying necessary values from $\omega$ one by one. Let $J_\CA$ be the set of $\omega_i$ that are queried to determine $f(\omega)$. As in (\cite{garban2014noise}) p $90$, we make the following definition 
    \begin{definition}
        The revealment of a randomised algorithm $\CA$ for a Boolean function $f$, denoted $\delta_A$, is defined by \[\delta_\CA = \max_{j\in {1,\hdots,2^{n+1}-1}} \Pb(j \in J_\CA)\]
        
        and the revealment of a Boolean function $f$ is defined as \[\delta_f = \inf_\CA \delta_\CA. \]
    \end{definition}

    The following crucial fact holds \cite{garban2014noise} p $93$. 
    
    \begin{theorem}
    \label{revealment}
        If the revealments satisfy
        
        \[\lim_{n \rightarrow \infty} \delta_{f_n} = 0\]
        
        then $\{f_n\}$ is noise sensitive. 
    \end{theorem} 
    
    We can now state the following theorem.
    
    \begin{theorem}
        \label{th:3-maj_stride_2}
        The sequence of convolutional iterated $3$-majority function on $G^{(2)}_n$ with stride $2$ is noise sensitive. This also holds on $G^{(2')}_n$.
    \end{theorem}
    \begin{proof}
    
    Let $n$ be fixed, and for now, a multiple of three. We recursively define an algorithm $\CA(m,W)$ that for all $m = 0,\hdots ,n/3$ and all $W \subseteq L_{3m}$, finds $f(w)$, $w \in W$ in random order. The algorithm goes as follows. 

{\bf Starting step:} $\CA(0,W)$: Query all bits in $W \subseteq L_0$ in a random uniform order.

\smallskip

{\bf Inductive step:} $\CA(m+1,W)$: Let $o$ be a uniform random permutation of $W$. Now 
recursively find $f(o(\ell))$, $\ell =1,\ldots,|W|$ by querying nodes in $D_{o(\ell),3m}$ with $\CA(m,D_{o(\ell),3m})$ with the modification that when and if, in the process of doing so, encountering a node $v \in D_{o(\ell),3m}$ such that $f(v)$ has no information to give on $f(o(\ell))$, then skip the query of $v$. When querying each $f(o(\ell))$, do not use any information gained when querying $f(o(1)),\ldots,f(o(\ell-1))$.

In other words, by this definition a node $v\in D_{W,3m}$ that has two ancestors, $a_1$ and $a_2$, in $W$ may have been queried to find $f(a_1)$ previously, but the knowledge of $f(v)$ is not used if one later needs to query $f(a_2)$ until the turn comes to $v$ in $\CA(m,D_{a_2,m})$ (if $f(v)$ can influence $f(a_2)$ at that point.) We refer to this as the algorithm is {\em forgetful}; when querying each $f(o(\ell))$, no input is known at the start of doing that. (Of course, when the turn comes to $v$, do not query $f(v)$ again, but simply do not use the knowledge of $f(v)$ until this precise moment.)

\smallskip

Note that by forgetfulness, the permutation $o$ is independent of which input bits are queried in the end. Refer to this as the algorithm being {\em permutation independent}.

\smallskip

Let $R_{m,W,j}$ be the event that $\CA(m,W)$ queries $j$ and let $q_m=\max_W\max_j \Pb(R_{m,W,j})$.

\smallskip 

Fix an arbitrary leaf $j$ and assume that $W \subseteq L_{3(m+1)}$ and $W \cap A_{j,3(m+1)} \neq \emptyset$.
We claim that $\Pb(R_{m+1,W,j}) \leq \Pb(R_{m,D_{W,3m},j})$. This follows directly from the forgetfulness of $\CA(m+1,W)$ which implies that since the full algorithm $\CA(m,D_{W,3m})$ queries every node in $D_{W,3m}$ whereas its modification when applied in $\CA(m+1,W)$ generally does not and then the forgetfulness implies that deterministically, for every possible input $\omega$, the modification only queries a subset of the input nodes of the ones queried by the full algorithm.  

In order to ever query $j$ by $\CA(m+1,W)$ it must be the case that
\begin{itemize}
    \item[(a)] At least one $a \in A_{j,3m}$ is queried when querying $w$ for some $w \in A_{a,3(m+1)}$ and
    \item[(b)] for at least one such $a$, $j$ is queried for finding $f(a)$.
\end{itemize}

For an ancestor $a \in A_{j,3(m+1)}$, let $o'_a$ be the permutation of $D_{a,3m}$ that is included by recursion in Algorithm $\CA(m,D_{a,3m})$. Let $E_a$ be the event that $o'_a(u) < o'_a(h) < o'_v(v)$ for all $u \in D_{a,3m} \setminus A^+_{j,3m}$, $h \in A_{j,3m} \cap D_{a,3m}$ and $v \in A^+_{j,3m} \cap D_{a,3m}\setminus A_{j,3m}$. Let $E = \bigcap_{a \in A_{j,3(m+1)}}E_a$.

\smallskip 

Let us lower bound $\Pb(E_a)$. Now $D_{a,3m}$ can contain either one or two ancestors of $j$ and $A^+_{j,3m}$ could be included in $D_{a,m}$ or one node in $A^+_{j,3m} \setminus A_{j,3m}$ could be outside $D_{a,m}$. The ``worst case'' for the present purpose is when
$j$ has two ancestors, $h_1,h_2 \in A_{j,3m}$ and $A^+_{j,3m} \subset D_{a,3m}$. Since $D_{a,3m}$ contains $15$ nodes we get in this case  
\[\Pb(E_a) = \frac{1}{6\binom{15}{4}}.\]
It is easy to see that in the other cases, $\Pb(E)$ is larger than the right hand side. Since $W$ may contain two ancestors of $j$, we get
\[\Pb(E) \geq \frac{1}{\left(6\binom{15}{4}\right)^2}.\]
Since the fact that probability of $\CA(m,D_{W,3m})$ querying $j$ is independent of of the order in which the nodes in $D_{W,3m}$ are queried (i.e.\ the permutation independence) and the above claim,
\[\Pb(R_{m+1,W,j}|E^c) \leq \Pb(R_{m,D_{W,3m},j})\leq q_m.\]
Let $F$ be the event that $f(u)=f(v)$ for all $u,v \in D_{A_{j,3(m+1)},3m} \setminus A^+_{j,3m}$. Then, according to Lemma \ref{lem:stride2}, $\Pb(F)\geq1/2^{18}$. On $E$, for $j$ to be queried, $F^c \cap R_{m,D_{W,3m},j}$ must necessarily occur. Also, $F$ and $R_{m,D_{W,3m},j}$ are independent. This holds since $R_{m,D_{W,3m},j}$ by forgetfulness depends only on the leafs in $D_{A_{j,3m}}$ and the randomness in the implicit random permutations when $\CA(m,a)$, $a \in A_{j,3m}$ are carried out, whereas $F$ depends on none of that. Clearly the two events are also independent of $E$. Hence we get
\[\Pb(R_{m+1,W,j}) \leq \Pb(E^c)q_m + \Pb(E)\Pb(F^c)\Pb(R_{m,D_{W,3m},j}|E) \leq (\Pb(E^c)+\Pb(E)\Pb(F^c))q_m.\]
Since this bound is independent of $W$ and $j$,
\[q_{m+1} \leq (\Pb(E^c)+\Pb(E)\Pb(F^c))q_m \leq \left(1 -\frac{1}{36 \binom{15}{4}^2 2^{18}}\right) q_m.\]

Hence 
\[q_{n/3} \leq \left(1 -\frac{1}{36 \binom{15}{4}^2 2^{18}}\right)^{n/3} \rightarrow 0\]
as desired.
Cases when $n$ is not a multiple of 3 can be completed by e.g.\ querying everything in layers $0,\ldots,3(n/3-\lfloor n/3 \rfloor)$ and then using the above algorithm from there. This gives
\[q_{n/3} \leq  \left(1 -\frac{1}{36 \binom{15}{4}^2 2^{18}}\right)^{\lfloor n/3 \rfloor} \rightarrow 0\]

which in combination with Theorem \ref{revealment} proves the theorem.
    
\end{proof}

The arguments in the proof are easily adapted to when $\theta_t$ are random and we get the following corollary.

\begin{corollary} \label{col:stride2_random}
    The iterated 3-majority function with stride $2$ is annealed and quenched noise sensitive if the distribution of $\theta_t$ is such that the probability of representing a majority is bounded away from zero.
\end{corollary}
\begin{proof}
    We show that with a very high probability, $\{f_n\}$ will be such that the revealment converges to zero as $n$ grows. As previously discussed, $\theta_t$ represents either a majority or dictator function.
    Using the same algorithm as in Theorem \ref{th:3-maj_stride_2}, $q_m$ is now random depending on the filter structure, and we get that for a fixed iteration $m$ 
    \[q_{m+1} \leq \left( 1 - \frac{1}{36 \binom{15}{4}^22^{18}}  \right) q_m\]
    if the layers $3m$, $3m+1$ and $3m + 2$ all represent majority layers. 
    If this is not the case, i.e.\ if at least one of the layers represents a dictator, we instead just use that $q_{m+1} \leq q_m$. 
    Letting 
    \[C_n = \left|\left\{m\in \{1,\hdots \lfloor n/3 \rfloor\} : \theta_{3m}, \theta_{3m+1} \textit{ and } \theta_{3m+2} \textit{ represents majority functions.}\right\}\right|\]
    we get
    \[q_{n/3} \leq  \left(1 -\frac{1}{36 \binom{15}{4}^2 2^{18}}\right)^{C_n}.\]
    Since the probability of a given $\theta_t$ representing a majority is bounded away from zero, $C_n \rightarrow \infty$ in probability as $n$ grows and hence the right hand side converges to 0. This proves quenched noise sensitivity. Finally, using Theorem \ref{th:relations_between_annealed_and_quenched} and the fact that $E_{\omega}[f_n(\omega)] = 0$ for all $\theta_t$, annealed noise sensitivity follows. 

\end{proof}
     
\subsection{Extensions to convolutional iterated $2k+1$ majority with overlap} \label{sec:extensions}
    
    Obviously the problems just treated for convolutional iterated 3-majority are equally interesting with 3 replaced with $2k+1$ for some integer $k \geq 2$. In this case the sensitivity is trivial if the stride $s > 2k+1$ since this would correspond to the regular $2k+1$ iterated majority with no charred nodes. For smaller $s$ the corresponding graph $G_{n,k,s}$ is defined as
    $G_{n,k,s} = (V_{n,k,s},E_{n,k,s})$, where 
        
        \begin{align*}
         V_{n,k,s} &= \{v_{n,0},v_{n-1,-k},v_{n-1,-(k-1)},\hdots,v_{n-1,k},v_{n-2,-2k},\hdots,v_{n-2,2k},\hdots ,v_{0,-kn},\hdots,v_{0,kn}\} \text{ and }\\
         E_{n,k,s} &= \{(v_{t,i},v_{t-1,j}) : t = n,\hdots,1,|i-j| \leq k ,v_{t,i} \text{ and }v_{t-1,j} \text{ exist}\}
        \end{align*}

    \noindent if $s = 1$ and 

    \begin{align*}
         V_{n,k,s} &= \{v_{n,1},v_{n-1,1},v_{n-1,2},\hdots, v_{n-1,2k+1}, v_{n-2,1}, \hdots , v_{n-2,\frac{2k s^2 - 2k + s -1}{(s-1)}}, \hdots, v_{0,1},\hdots,v_{0,\frac{2k s^n - 2k + s -1}{(s-1)}}\} \text{ and }\\
         E_{n,k,s} &= \{(v_{t,i},v_{t-1,j}) : t = 1,\hdots,n, j-s(i-1) \in [1,2k+1],v_{t,i} \text{ and }v_{t-1,j} \text{ exist}\}
        \end{align*}

    \noindent if $1 < s < 2k+1$.
    
    When $s > 1$, the crucial condition for the algorithm in Theorem \ref{th:3-maj_stride_2} for proving noise sensitivity is that each input node has at most two ancestors at any generation. This can in fact be generalised to the $2k+1$ iterated majority with stride $s \geq 2$. This is shown in Lemma \ref{lem:number_anc}.
    
    \begin{lemma}
    \label{lem:number_anc}
    Let $G_{n,k,s}$ be the corresponding graph to the convolutional iterated $2k+1$-majority with stride $s \geq 2$. Then $|A_{t,S}| \leq 2k$, for every $t$ and every $S$ of $m \leq 2k$ nodes next to each other. 
    \end{lemma}
    \begin{proof}
        Notice that if $s > 2k$, the corresponding graph $G_{n,k,s}$ would be such that $|A_{t,j}| = 1$ for each $t$ and $j$, making the result trivial. Therefore, assume $2 \leq s \leq 2k$. 
        Now observe that given a node $v_{t,i}$ at layer $t > 0$ has $2k +1$ children which are $\{v_{t-1,s(i-1) + j} : j \in \{1,\hdots 2k+1\}\}$.
        So, for a node $v_{t+1,i}$ to be a parent to $v_{t,j}$ it must be that $j = s(i-1) + i'$ for some $i' \in \{1,\hdots 2k+1\}$. Consequently, the parents to $v_{t,j}$ is $v_{t+1,i}$ such that 
        \[\left\lceil \frac{j-2k}{s} \right\rceil + 1 \leq i \leq \left\lfloor \frac{j-1}{s} \right\rfloor + 1\] 
        such that $v_{t,i}$ exists. Let $S=\{v_{t,u},v_{t,u+1},\hdots,v_{t,u+m}\}$ where $u > 0$ and $0 \leq m \leq 2k-1$ such that $v_{t,u}$ and $v_{t,u+m}$ exist. Then 
        \begin{align*}
            |A_{t+1,S}| = \left\lfloor \frac{u+m-1}{s} \right\rfloor - \left\lceil \frac{u-2k-1}{s} \right\rceil + 1 \leq \left\lfloor \frac{u+2k-2}{2} \right\rfloor - \left\lceil \frac{u-2k-1}{2} \right\rceil + 1 = 2k
        \end{align*}
        
        The lemma now follows by using this argument recursively over $t$. 
        
    \end{proof}

    From Lemma \ref{lem:number_anc} it follows  that $|A_t,j| \leq 2k$ for all $t$ and $j$. 
    With that condition satisfied, the proof of Theorem \ref{th:3-maj_stride_2} easily generalises. Without supplying further details, we state the following theorem.
    
    \begin{theorem}
    The convolutional iterated $2k+1$-majority function with stride $s$ is noise sensitive if $s \geq 2$.
    \end{theorem}

    The results for noise stability also go through very easily if the stride is $1$ when $\theta_t$ are non-random. This by noticing that if there are $k+1$ sequentially equal input bits with the same sign, no information can be transferred from one side to the other. In the non cyclic case, input bits "outside" such a sequence has no influence on the final result. So by a high probability, the function is determined by the central bits, which with a high probability does not have a disagreement between $\omega$ and $\omega^\epsilon$. 
    
    A similar result can be stated for the cyclic case where the proof unfolds in a similar fashion but with a slightly different definition on $\tau_j$. All in all, we can state the following theorem.
    
    \begin{theorem}
    The convolutional iterated $2k+1$-majority function with stride $1$ is noise stable both with and without cyclical convention.
    \end{theorem}

    When considering random weights 
    for $s \geq 2$, noise sensitivity still holds as long as the distribution of $\theta_t$ is such that with probability bounded away from $0$, $\theta_t$ represents ordinary majority. The arguments are identical to that of $k=1$. Thus, we can state the following corollary.
    \begin{corollary}
        If $\theta_t$ is such that the probability of representing a majority is bounded away from zero, then the convolutional iterated $2k+1$-majority function with stride $s \geq 2$ and random weights $\theta_t$ is annealed and hence also quenched noise sensitive. 
    \end{corollary}

    {\bf Remark.} For stride one, the arguments with random weights do not generalise well when $k > 1$.

\section{Open problems and research directions} \label{sec:conclusion}

    We have taken inspiration from observed non-robustness phenomena for DNN classifiers and have determined when a few common DNN architectures give rise to noise sensitive or noise stable classifiers. Going on focusing on the non-robustness phenomena, there are numerous further avenues to be explored. First of all, of course, there are many other DNN architectures that can be considered. Can noise sensitivity/stability results be achieved for (some of) them? If so, will that help us to design powerful DNN architectures that are robust to noise?

\smallskip

Besides that, here are a first few questions concerning the fully connected DNN models.

\begin{itemize}
    \item In this paper, the activation function is the sign function for all layers. Will replacing them with activation functions used in practice, such as the arctan function, the sigmoid function, the ReLU function, etc, cause different properties with respect to noise sensitivity/stability?
    
    \item Above, all layers are of equal width. Does it make a difference if the layers are allowed to be of different widths? Does it then matter how unequal the widths are, e.g.\ of vastly different orders as functions of $n$? Here we believe that having varying widths, does not matter no matter how much they vary, and that would be a fairly easy task to prove.  
    
    \item In Section \ref{sec:correlated} it is shown that under weak assumptions, if $1-\rho_n$ shrinks as $\log n/n$ or faster, then the resulting Boolean function will almost certainly be noise stable. On the other hand, if $1-\rho_n$ shrinks as $(\log n)^3/n$, then we typically get a noise sensitive sequence of functions. Is there a cutoff between the two cases and if so, where between $\log n/n$ and $(\log n)^3/n$ is it?
    
    \item What happens if one has access to data and train the network to fit with that? This could be interpreted in many different ways. For example the setting could be the following. Suppose that data are generated by a particular fully connected DNN of the kind in Section \ref{sec:uncorrelated} and suppose that this particular DNN expresses a noise stable function. Assume further that we get data-points one by one, where each data-point is a uniform random input together with its output. According to Theorem \ref{th:noise_sensitivity_deep_networks}, the network will almost certainly be noise sensitive at the start of training and in the end, after having seen all possible inputs, the DNN has learnt to express the true noise stable function behind the data. Where along this process does the DNN turn from sensitive to stable? Can it also flip back and forth between noise sensitive and noise stable along the way? 
    
\end{itemize}

Concerning the convolutional models, there is also more to be done for $k \geq 2$, i.e.\ for filters of size at least $5$. We have already observed that when the filter size is at least $5$, a filter can express many things besides (anti)majority and (anti)dictator function. For example with five input bits, a filter could express ``the first input bit unless all the other bits agree on the opposite''.
For stride $s \geq 2$, we showed that if the filter with at least some probability expresses a majority, then the resulting network is noise sensitive. However if regular majority is never expressed, then what? 

Also, for $k \geq 2$ and stride 1, it remains open if the resulting network is noise stable as we believe it is.

\smallskip
\smallskip

The noise sensitivity properties in Section \ref{sec:uncorrelated} and \ref{sec:correlated} are very strong and one can even input two strings from any joint distribution over $\{-1,1\}^n \times \{-1,1\}^n$ such that with probability tending to 1 with $n$ the two strings neither equal nor completely opposite. With probability $1/2$, the function described by the resulting DNN will produce different results for the two different strings. However, this is in a nonadversarial setting, i.e.\ the joint distribution of the input strings is independent of the network weights. 
Suppose an adversary gets information about the network weights and one of the input strings. Can he produce the second string by flipping bits of his own choice of the first string, with no other restriction than the expected number of flips has to be very small, so that the output of the second string differs from that of the first string? 
None of the results in this paper are concerned with questions such as these, and it would certainly be very interesting to look at that.

\section*{Achnowledgement}
The first and third authors were supported by the Wallenberg AI, Autonomous Systems and Software Program (WASP) funded by the Knut and Alice Wallenberg Foundation. The second author acknowledges the support of the Swedish Research Council, grant no. 2020-03763.

\bibliographystyle{abbrv}
\bibliography{references}

\end{document}